\documentclass[10pt,journal,compsoc]{IEEEtran}
\ifCLASSOPTIONcompsoc
  \usepackage[nocompress]{cite}
\else
  \usepackage{cite}
\fi

\usepackage{algorithm, algpseudocode}
\usepackage{multirow}
\usepackage{amsmath}
\usepackage{amsfonts}
\usepackage{enumitem}
\usepackage{graphicx}
\graphicspath{ {./images/} }
\usepackage{makecell}
\usepackage{diagbox}
\usepackage{ctable}
\usepackage{ragged2e}
\usepackage{soul}
\usepackage{url}

\algnewcommand{\Inputs}[1]{%
  \State \textbf{Inputs:}
  \Statex \hspace*{\algorithmicindent}\parbox[t]{.8\linewidth}{\raggedright #1}
}
\algnewcommand{\Initialize}[1]{%
  \State \textbf{Initialize:}
  \Statex \hspace*{\algorithmicindent}\parbox[t]{.8\linewidth}{\raggedright #1}
}
\usepackage{amsthm}
\newtheorem{theorem}{Theorem}[section]

\newtheorem{proposition}[theorem]{Proposition}
\newtheorem{lemma}[theorem]{Lemma}
\newtheorem{definition}[theorem]{Definition}
\ifCLASSINFOpdf
\else
\fi
\newcommand*{\xdash}[1][0.5em]{\rule[0.5ex]{#1}{0.55pt}}

\begin{document}
%
% paper title
% Titles are generally capitalized except for words such as a, an, and, as,
% at, but, by, for, in, nor, of, on, or, the, to and up, which are usually
% not capitalized unless they are the first or last word of the title.
% Linebreaks \\ can be used within to get better formatting as desired.
% Do not put math or special symbols in the title.
\title{WOOD: Wasserstein-based Out-of-Distribution Detection}
%
%
% author names and IEEE memberships
% note positions of commas and nonbreaking spaces ( ~ ) LaTeX will not break
% a structure at a ~ so this keeps an author's name from being broken across
% two lines.
% use \thanks{} to gain access to the first footnote area
% a separate \thanks must be used for each paragraph as LaTeX2e's \thanks
% was not built to handle multiple paragraphs
%
%
%\IEEEcompsocitemizethanks is a special \thanks that produces the bulleted
% lists the Computer Society journals use for "first footnote" author
% affiliations. Use \IEEEcompsocthanksitem which works much like \item
% for each affiliation group. When not in compsoc mode,
% \IEEEcompsocitemizethanks becomes like \thanks and
% \IEEEcompsocthanksitem becomes a line break with idention. This
% facilitates dual compilation, although admittedly the differences in the
% desired content of \author between the different types of papers makes a
% one-size-fits-all approach a daunting prospect. For instance, compsoc 
% journal papers have the author affiliations above the "Manuscript
% received ..."  text while in non-compsoc journals this is reversed. Sigh.

\author{Yinan Wang, Wenbo Sun, Jionghua (Judy) Jin, Zhenyu (James) Kong, Xiaowei Yue*
\IEEEcompsocitemizethanks{\IEEEcompsocthanksitem Yinan Wang, James Kong, and Xiaowei Yue are with the Grado Department of Industrial and Systems Engineering, Viriginia Tech, Blacksburg, VA, 24060.\protect\\
% note need leading \protect in front of \\ to get a newline within \thanks as
% \\ is fragile and will error, could use \hfil\break instead.
E-mail: \{yinanw, zkong, xwy\}@vt.edu
\IEEEcompsocthanksitem Wenbo Sun, and Jionghua (Judy) Jin are with the Department of Industrial and Operations Engineering, University of Michigan, Ann Arbor, MI, 48109.\protect\\
% note need leading \protect in front of \\ to get a newline within \thanks as
% \\ is fragile and will error, could use \hfil\break instead.
E-mail: \{sunwbgt, jhjin\}@umich.edu}% <-this % stops an unwanted space

\thanks{* Corresponding Author: Xiaowei Yue, (xwy@vt.edu)}}

% note the % following the last \IEEEmembership and also \thanks - 
% these prevent an unwanted space from occurring between the last author name
% and the end of the author line. i.e., if you had this:
% 
% \author{....lastname \thanks{...} \thanks{...} }
%                     ^------------^------------^----Do not want these spaces!
%
% a space would be appended to the last name and could cause every name on that
% line to be shifted left slightly. This is one of those "LaTeX things". For
% instance, "\textbf{A} \textbf{B}" will typeset as "A B" not "AB". To get
% "AB" then you have to do: "\textbf{A}\textbf{B}"
% \thanks is no different in this regard, so shield the last } of each \thanks
% that ends a line with a % and do not let a space in before the next \thanks.
% Spaces after \IEEEmembership other than the last one are OK (and needed) as
% you are supposed to have spaces between the names. For what it is worth,
% this is a minor point as most people would not even notice if the said evil
% space somehow managed to creep in.

% The paper headers
\markboth{Manuscript}%
{Shell \MakeLowercase{\textit{et al.}}: Bare Demo of IEEEtran.cls for Computer Society Journals}
% The only time the second header will appear is for the odd numbered pages
% after the title page when using the twoside option.
% 
% *** Note that you probably will NOT want to include the author's ***
% *** name in the headers of peer review papers.                   ***
% You can use \ifCLASSOPTIONpeerreview for conditional compilation here if
% you desire.

% The publisher's ID mark at the bottom of the page is less important with
% Computer Society journal papers as those publications place the marks
% outside of the main text columns and, therefore, unlike regular IEEE
% journals, the available text space is not reduced by their presence.
% If you want to put a publisher's ID mark on the page you can do it like
% this:
%\IEEEpubid{0000--0000/00\$00.00~\copyright~2015 IEEE}
% or like this to get the Computer Society new two part style.
%\IEEEpubid{\makebox[\columnwidth]{\hfill 0000--0000/00/\$00.00~\copyright~2015 IEEE}%
%\hspace{\columnsep}\makebox[\columnwidth]{Published by the IEEE Computer Society\hfill}}
% Remember, if you use this you must call \IEEEpubidadjcol in the second
% column for its text to clear the IEEEpubid mark (Computer Society jorunal
% papers don't need this extra clearance.)

% use for special paper notices
%\IEEEspecialpapernotice{(Invited Paper)}

% for Computer Society papers, we must declare the abstract and index terms
% PRIOR to the title within the \IEEEtitleabstractindextext IEEEtran
% command as these need to go into the title area created by \maketitle.
% As a general rule, do not put math, special symbols or citations
% in the abstract or keywords.
\IEEEtitleabstractindextext{%
\begin{abstract}
\justifying
The training and test data for deep-neural-network-based classifiers are usually assumed to be sampled from the same distribution. When part of the test samples are drawn from a distribution that is sufficiently far away from that of the training samples (a.k.a. out-of-distribution (OOD) samples), the trained neural network has a tendency to make high confidence predictions for these OOD samples. Detection of the OOD samples is critical when training a neural network used for image classification, object detection, etc. It can enhance the classifier's robustness to irrelevant inputs, and improve the system resilience and security under different forms of attacks. Detection of OOD samples has three main challenges: (i) the proposed OOD detection method should be compatible with various architectures of classifiers (e.g., DenseNet, ResNet), without significantly increasing the model complexity and requirements on computational resources; (ii) the OOD samples may come from multiple distributions, whose class labels are commonly unavailable; (iii) a score function needs to be defined to effectively separate OOD samples from in-distribution (InD) samples. To overcome these challenges, we propose a Wasserstein-based out-of-distribution detection (WOOD) method. The basic idea is to define a Wasserstein-distance-based score that evaluates the dissimilarity between a test sample and the distribution of InD samples. An optimization problem is then formulated and solved based on the proposed score function. The statistical learning bound of the proposed method is investigated to guarantee that the loss value achieved by the empirical optimizer approximates the global optimum. The comparison study results demonstrate that the proposed WOOD consistently outperforms other existing OOD detection methods.

\end{abstract}

% Note that keywords are not normally used for peerreview papers.
\begin{IEEEkeywords}
OOD detection, Wasserstein distance, machine learning, image classification, cyber security
\end{IEEEkeywords}}

% make the title area
\maketitle

% To allow for easy dual compilation without having to reenter the
% abstract/keywords data, the \IEEEtitleabstractindextext text will
% not be used in maketitle, but will appear (i.e., to be "transported")
% here as \IEEEdisplaynontitleabstractindextext when the compsoc 
% or transmag modes are not selected <OR> if conference mode is selected 
% - because all conference papers position the abstract like regular
% papers do.
\IEEEdisplaynontitleabstractindextext
% \IEEEdisplaynontitleabstractindextext has no effect when using
% compsoc or transmag under a non-conference mode.

% For peer review papers, you can put extra information on the cover
% page as needed:
% \ifCLASSOPTIONpeerreview
% \begin{center} \bfseries EDICS Category: 3-BBND \end{center}
% \fi
%
% For peerreview papers, this IEEEtran command inserts a page break and
% creates the second title. It will be ignored for other modes.
\IEEEpeerreviewmaketitle

\IEEEraisesectionheading{\section{Introduction}\label{sec:introduction}}
% Computer Society journal (but not conference!) papers do something unusual
% with the very first section heading (almost always called "Introduction").
% They place it ABOVE the main text! IEEEtran.cls does not automatically do
% this for you, but you can achieve this effect with the provided
% \IEEEraisesectionheading{} command. Note the need to keep any \label that
% is to refer to the section immediately after \section in the above as
% \IEEEraisesectionheading puts \section within a raised box.

% The very first letter is a 2 line initial drop letter followed
% by the rest of the first word in caps (small caps for compsoc).
% 
% form to use if the first word consists of a single letter:
% \IEEEPARstart{A}{demo} file is ....
% 
% form to use if you need the single drop letter followed by
% normal text (unknown if ever used by the IEEE):
% \IEEEPARstart{A}{}demo file is ....
% 
% Some journals put the first two words in caps:
% \IEEEPARstart{T}{his demo} file is ....
% 
% Here we have the typical use of a "T" for an initial drop letter
% and "HIS" in caps to complete the first word.
\IEEEPARstart{D}{eep} Neural Networks (DNNs) have achieved outstanding performances on many challenging tasks, such as image classification \cite{NIPS2012_c399862d, 7780459}, object detection \cite{8237584, 7485869}, and speech recognition \cite{6857341}. When training a DNN, it is often assumed that the training and test samples are drawn from the same distribution. However, in practice, there likely exist abnormal test samples that are drawn from other distributions. Comparing to the samples drawn from the distribution of training samples (a.k.a. in-distribution (InD) data), these abnormal samples are referred to as out-of-distribution (OOD) data, which may not belong to any of the classes that the model is trained on. In this situation, a DNN-based classifier tends to over-confidently predict OOD samples with the class labels of InD samples. In addition, OOD samples in the training dataset may have a significant impact on the learning performance, which typically results in a reduction of the classification accuracy and diagnosis capability. Therefore, it is an urgent need to strengthen DNNs with the ability to detect OOD samples.

OOD detection for a DNN-based classifier mainly includes two sub-tasks: (i) A score function needs to be properly defined to evaluate the difference between InD and OOD samples; (ii) The InD and OOD samples need to be further separated in the space defined by the score function. These two sub-tasks are coupled with each other. A well-defined score function can better capture the difference and naturally separate InD and OOD samples, and such difference can be further enlarged by various techniques, such as preprocessing the input, calibrating the hyperparameters used in the score function, and retraining the classifier. 

In this paper, we propose the Wasserstein-based OOD detection (WOOD) method, which is an unified framework to train a DNN-based classifier to conduct classification and OOD detection simultaneously. The ``Wasserstein'' here refers to the Wasserstein distance \cite{ruschendorf1985wasserstein}, which is designed to evaluate the difference between probability distributions. In the following section, we first categorize and review some recent studies in OOD detection, and then discuss the necessity and current research gaps in adapting the Wasserstein distance as a dissimilarity measure and a loss function for OOD detection.

\subsection{Related Work}

\subsubsection{OOD Detection Without Tuning the Pre-trained Classifier}
\label{sec:cate1}
The recently developed methods of OOD detection in classification can be summarized into two categories according to whether the classifier is re-trained on the OOD samples. In the first category, OOD detection is regarded as an augmented property for a pre-trained classifier. For example, Hendrycks and Gimpel built a baseline for OOD detection using the threshold-based detector and the maximum softmax score as the score function \cite{DBLP:conf/iclr/HendrycksG17}. Liang et al. further rescaled the softmax score with a temperature parameter and preprocessed the inputs to further enlarge the difference between the maximum softmax score of InD and OOD samples \cite{liang2020enhancing}. Rather than relying on the maximum softmax score, some researchers tried to define the score functions based on different distance measures. For example, the Mahalanobis distance was calculated and calibrated on the intermediate features of DNNs to serve as the confidence score \cite{lee2018simple, denouden2018improving}; A measure of confidence was proposed by analyzing the invariance of softmax score under various transformations of inputs \cite{bahat2018confidence}; the uncertainty of DNNs was evaluated by using the gradient information from all the layers to serve as the score function \cite{oberdiek2018classification}; the trust score for each input was defined as the ratio of the Hausdorff distances from the input to its closest and second closest labels, which is used to determine whether a classifier's prediction can be trusted or not \cite{10.5555/3327345.3327458}. By projecting the inputs into a new space, these newly defined score functions can distinguish the InD and OOD samples better than the methods relying on the softmax score. 

Although these methods are compatible with different neural network architectures and are easy to implement, the score-function-based OOD detection methods have two limitations: (i) The choice of score functions depends on the type of applications and OOD patterns. It is required to design a score function that has a consistently good performance under all OOD patterns. (ii) The OOD detection performance highly relies on the values of hyperparameters in the selected score function, which are often tuned based on auxiliary OOD samples. However, after the classifier has been pre-trained, simply tuning the hyperparameters in the score function may not guarantee a good separation of  InD and OOD samples, which may not achieve a desired OOD detection performance.

\subsubsection{OOD Detection by Re-training Classifiers}
\label{sec:cate2}
In the second category, researchers tried to propose a unified framework to simultaneously train an OOD detector during the training process of the neural-network-based classifier. For example, Hendrycks et al. \cite{hendrycks2019oe} proposed an outlier exposure method to leverage auxiliary datasets of OOD outliers for training classifiers . An adversarial training with informative outlier mining (ATOM) method was proposed to improve the robustness of OOD detection with auxiliary datasets \cite{chen2021atom}. Moreover, DeVries and Taylor \cite{devries2018learning} proposed to add a confidence branch on the top of a classifier and train this branch to generate binary outputs to represent whether one input sample is OOD or not. Mohseni et al. \cite{Mohseni_Pitale_Yadawa_Wang_2020} added an auxiliary head on the top of a classifier and trained it in the self-supervised way. 

These methods can identify some fundamental differences between InD and OOD samples when training the classifier. However, they have three limitations: (i) The OOD detectors in these methods are based on the maximum softmax score or a separate branch of binary classifiers. The distance measure cannot represent the difference between two discrete distributions accurately; (ii) Additional components added to a classifier will increase the model complexity and training time; (iii) The auxiliary OOD dataset is often unlabelled. It is also unrealistic to assume all the OOD samples are from the same class. In practice, human beings usually categorize one sample to be OOD when it has significant differences from InD samples and does not belong to any of the InD classes. By forcing different types of OOD samples to fall in the same class, this category of OOD detection methods brings additional difficulty in training the classification model, which may result in poor performances in OOD detection and InD classification.

\subsubsection{Reviews of distance metrics for probability distribution}

The softmax output from a classifier for a specific input is the discrete probability distribution over all the possible classes. Thus, the dissimilarity between output distributions directly reflects the difference between inputs. Wasserstein distance is often used to measure the dissimilarity between probability distributions. Compared with the Kullback-Leibler (KL) divergence and Jensen-Shannon (JS) divergence, Wasserstein distance provides a meaningful and smooth representation of the distance even when two distributions are located in lower-dimensional manifolds without overlaps \cite{weng2019gan}. The smooth measure of distance ensures a stable learning process when optimizing the Wasserstein distance with the gradient-based method. Wasserstein distance has demonstrated its strength in many applications, such as the distributionally robust stochastic optimization (DRSO) \cite{gao2016distributionally}, deep active learning \cite{Shui2020DeepAL}, reinforcement learning \cite{NEURIPS2019_f8363057}, image classification \cite{NIPS2015_a9eb8122}, and data augmentation \cite{Li2021}. 

In this paper, we use Wasserstein distance to measure the dissimilarity between the InD and OOD samples, which is further used to detect OOD samples. To the best of our knowledge, our proposed WOOD is the first work to design the OOD detection framework by leveraging the strength of Wasserstein distance in measuring the dissimilarity between distributions. Although there are prior works adapting Wasserstein distance into the loss function of a classification task \cite{NIPS2015_a9eb8122}, and designing the Generative and Adversarial Networks (GAN) \cite{pmlr-v70-arjovsky17a}, research gaps still exist in designing OOD detection framework with Wasserstein distance: (1) Unlike the classification task, the label information of OOD samples are commonly unavailable, which make the Wasserstein-based loss function in the classification task invalid to be used in OOD detection; (2) Instead of a single objective in previous works, the proposed WOOD framework is required to train the model to do classification and OOD detection simultaneously.

\subsection{Motivations and Contributions}
\begin{figure*}[ht]
    \includegraphics[width=0.8\linewidth]{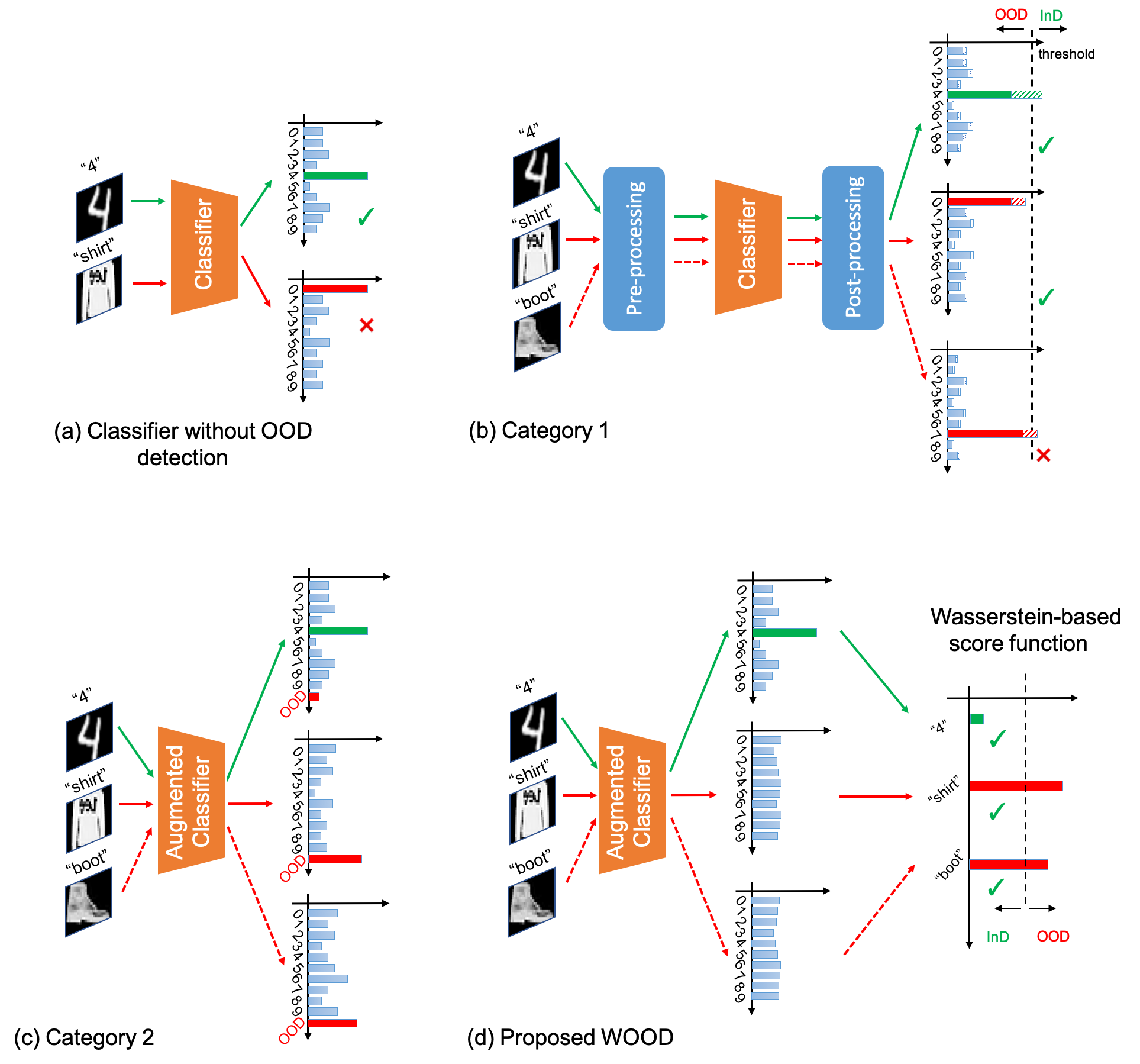}
    \centering 
    \caption{Comparison among the classifier without OOD detection, OOD detection without re-training the classifier, OOD detection by re-training the classifier, and the proposed WOOD. (a) The classifier without OOD detection tends to over-confidently assign incorrect labels to OOD samples; (b) One category of research works enable OOD detection only by pre-processing the input and post-processing the output; (c) Another category of research works enable OOD detection by training the classifier based on the assumption that all OOD samples are from the same class; (d) Our proposed WOOD specifically design Wasserstein-distance-based score function to evaluate the dissimilarities between InD and OOD samples and train the model by relaxing the assumption of OOD samples.}
    \label{fig:overall_flow}
\end{figure*}

To overcome the above limitations in using Wasserstein distance for OOD detection, we propose a new integrative approach, named WOOD, which integrates the above two categories of OOD detection methods (Sections \ref{sec:cate1} and \ref{sec:cate2}). Specifically, we design a new OOD score function whose hyperparameters can be simultaneously updated with the classifier. Fig. \ref{fig:overall_flow} shows a comparison of the classifier without OOD detection, the two categories of the existing OOD detection methods, and the proposed method. Going beyond the original neural-network-based classifier in Fig. \ref{fig:overall_flow} (a) \cite{NIPS2012_c399862d, 7780459}, Fig. \ref{fig:overall_flow} (b) demonstrates the flowchart of the first category of the existing methods, which maps inputs to a specific score function by preprocessing the inputs and postprocessing the outputs \cite{DBLP:conf/iclr/HendrycksG17} \xdash\cite{10.5555/3327345.3327458}. However, without simultaneously training the DNN for classification and OOD detection, it is often challenging to find a set of optimal hyperparameters to consistently identify various types of OOD samples. Fig. \ref{fig:overall_flow} (c) shows the idea of simultaneously training the classifier with InD and auxiliary OOD samples \cite{hendrycks2019oe, chen2021atom, devries2018learning, Mohseni_Pitale_Yadawa_Wang_2020}. The advantage of these methods is the efficiency of realizing sample classification and OOD detection simultaneously. However, they rely on the assumption that  all the OOD samples belong to the same class, which may not hold when different types of OOD samples are included in the test dataset.

Motivated by the aforementioned research gaps, as shown in Fig. \ref{fig:overall_flow} (d), we propose the Wassersterin-based score function and integrates it with the unified framework for OOD detection. The definition of our score function is intuitive. Suppose a classifier is trained to output the probabilities of a sample falling in each potential class. For an InD sample, we would like the corresponding prediction probabilities to be close to a one-hot vector, whose elements involve a close-to-one value for the true class, and close-to-zero values for all the other classes. For an OOD sample, we would like the prediction probabilities to be fuzzy among most classes, that is, a vector whose multiple elements taking similar values. 

In our proposed WOOD, Wasserstein distance \cite{ruschendorf1985wasserstein} is adopted to quantitatively measure such a distance between output discrete probability distributions when building the OOD score function, which is further incorporated into designing the multi-objective loss function. Wasserstein distance has two advantages: (i) it is defined in the metric space to evaluate the dissimilarity between distributions, which makes it a natural fit to the OOD detector; (ii) it can be optimized with the gradient-based method, so it is resilient to and compatible with various designs of loss function and diverse classifier structures. Comparing to the two categories of the existing OOD detection methods, the contributions of the proposed WOOD method are four folds:

\begin{enumerate}
    \item A Wasserstein-distance-based score function is proposed for the separation of InD and OOD samples. The definition of the score function is intuitive, and will be shown to outperform the benchmark score functions in several commonly-used datasets in the machine learning field.
    
    \item A unified framework is designed to simultaneously train the classifier and OOD detector. Such a one-step optimization setup will not introduce additional model complexity or computational load.
    
    \item The proposed OOD detection method does not require detailed labels of the OOD samples, nor force all the OOD samples to belong to the same class. With a relaxed assumption on the OOD samples, the proposed method can be generally applied to various applications.
    
    \item The theoretical properties of WOOD are investigated, which guarantees that the loss value achieved by the empirical minimizer approaches the global optimum. The gradient of the proposed loss function in WOOD is explicitly derived to further save computation time. 
\end{enumerate}

The remainder of this paper is organized as follows. Section \ref{framework} proposes the formulation of the WOOD method and discusses the properties and computational complexity of WOOD with two different cost matrices. Section \ref{learning} develops the learning algorithm of WOOD and derives the statistical learning bound of the proposed loss function. Section \ref{experiment} demonstrates the performance of WOOD using multiple datasets. Finally, Section \ref{conclusion} summarizes the contributions and concludes this paper.

\section{Proposed Wasserstein-based Out-of-Distribution Detection}
\label{framework}
We show the technical details of the proposed Wasserstein-based out-of-distribution detection (WOOD) method in this section. First, we briefly introduce the problem setup. Then, we formally define the OOD score function, incorporate it into the loss function, and formulate the corresponding detector for OOD samples. Lastly, we formulate the training of the OOD detector as an optimization problem and describe the procedure to obtain the numerical solution.

\subsection{Problem Setup}
In a classification problem, let $\mathcal{X}$ denote the space of inputs. Let $\mathcal{X}_\text{InD}$ and $\mathcal{X}_\text{OOD}$ denote two distinct sub-spaces on $\mathcal{X}$. Specifically, $\mathcal{X}_\text{InD}$ represents the space of InD samples with labels from $\mathcal{K}_\text{InD}$, and $\mathcal{X}_\text{OOD}$ represents the space of OOD samples whose labels lie beyond $\mathcal{K}_\text{InD}$. Here $\mathcal{X}=\mathcal{X}_\text{InD}\bigcup\mathcal{X}_\text{OOD}$. We refer to the input data from $\mathcal{X}_\text{InD}$ and $\mathcal{X}_\text{OOD}$ as In-distribution (InD) samples and Out-Of-Distribution (OOD) samples, respectively. Our objective is to train a classifier to correctly label the InD inputs from $\mathcal{X}_\text{InD}$, and a detector to identify samples from $\mathcal{X}_\text{OOD}$.

Our training dataset is organized as follows. Let $\mathcal{D}_\text{InD}$ denote the training dataset of InD samples. Each sample in $\mathcal{D}_\text{InD}$ consists of an input from $\mathcal{X}_\text{InD}$, and a corresponding class label from $\mathcal{K}_\text{InD}$. Since it is difficult to train an accurate OOD detector based on InD samples only \cite{chen2021atom}, we consider the case where the training dataset of auxiliary OOD samples, $\mathcal{D}_\text{OOD}$, is available. In real practice, $\mathcal{D}_\text{OOD}$ can be generated from adversarial training techniques \cite{10.5555/2969033.2969125} when real OOD samples are difficult to collect. Each element in the OOD dataset is an unlabelled input from $\mathcal{X}_\text{OOD}$ because OOD samples may come from various classes without exact labels.

Given the datasets described as above, our objective is two-folds. First, for an InD sample $x_\text{InD}\in\mathcal{X}_\text{InD}$ with label $k\in\mathcal{K}_\text{InD}$, we would like the predicted probability by the classifier to satisfy the following condition:
\begin{equation}
    f_{\theta}(x_\text{InD})\approx y^k,
\end{equation}
where $y^k$ is an one-hot vector whose $k^{th}$ element is $1$ and other elements are $0$'s; $f_{\theta}$ is a classifier that is parametrized by the model parameter $\theta$ and outputs the probabilities of an input $x_\text{InD}$ falling into all the class in $\mathcal{K}_{\text{InD}}$. Second, for an arbitrary sample $x\in\mathcal{X}$, we would like to train an OOD detector $g$ such that
\begin{equation}
    g(x)=
    \begin{cases}
    1,              & \text{if } x\in \mathcal{X}_\text{OOD}\\
    0,              & \text{if } x\in \mathcal{X}_\text{InD}.
    \end{cases}
\end{equation}

The true negative rate (TNR) and false negative rate (FNR) defined on the OOD detector $g$ are
\begin{align}
    \text{TNR}(g) = \mathbb{E}_{x \in \mathcal{X}_{\text{InD}}} \mathbb{I}[g(x)=0],
    \notag\\
    \text{FNR}(g) = \mathbb{E}_{x \in \mathcal{X}_{\text{OOD}}} \mathbb{I}[g(x)=0],
\end{align}
where $\mathbb{I}[.]$ is a sign function, when $g(x)=0$, its value is $1$, otherwise, its value is $0$. TNR denotes the percentage of InD samples that are correctly identified, and FNR is the percentage of OOD samples that are misidentified as the InD samples.

\subsection{Wasserstein-based score function}
We start from the general formulation of the Wasserstein distance for discrete distributions. Suppose $r_1$ and $r_2$ are discrete distributions of random variables that take values in $K$ classes with labels $\{1,...,K\}$, the Wasserstein distance between two $r_1$ and $r_2$ \cite{NIPS2015_a9eb8122} is

\begin{align}
     W(r_1,r_2) &= \text{inf}_{P\in \Pi(r_1,r_2)} \langle P,M \rangle,
    \label{Eq2-4}
\end{align}
where the $M \in \mathbb{R}^{K \times K}_{+}$ is the cost matrix; $\langle\cdot,\cdot\rangle$ represents Kronecker product between two matrices; $P$ is the joint probability distribution with an element of
\begin{align}
    \Pi(r_1,r_2):=\left\{P \in \mathbb{R}_{+}^{K \times K} \mid P \mathbf{1}_{K}=r_2, P^{\top} \mathbf{1}_{K}=r_1\right\},
    \label{Eq3-2}
\end{align}
where $\mathbf{1}_{K} \in \mathbb{R}^{K}$ denotes the $K$-dimensional vector whose elements are all ones; $r_{1}$ and $r_{2}$ are the marginal distributions of $P$. The properties of Wasserstein distance is included in Appendix \ref{Appendix-D}.

As was mentioned in Section~\ref{sec:introduction}, an OOD score function should indicate how likely an incoming sample belongs to $\mathcal{X}_\text{OOD}$. The Wasserstein distance is adapted to design the OOD score function by leveraging its advantage in evaluating the dissimilarity between distributions. Let $K$ denote the number of different classes in $\mathcal{K}_\text{InD}$. For $x\in\mathcal{X}$, the Wasserstein distance between the predictions given by $f_\theta$ and the one-hot vector $y^k$ of class $k$ is computed as:
\begin{align}
     W(f_{\theta}(x),y^k) &= \inf_{P} \langle P,M \rangle,
     \notag\\
    \text{s.t.}\;\;
    P\mathbf{1}_{K} &= y^k,
    \notag\\
    P^{\top}\mathbf{1}_{K} &= f_{\theta}(x),
    \notag\\
    k &\in \mathcal{K}_\text{InD},
    \label{Eq3-1}
\end{align}
where $P$ is a joint discrete probability distribution whose marginal distributions are $f_\theta(x)$ and $y^k$. It is worth noting that the feasible region of $P$ might change when calculating Wasserstein distances between different pairs of marginal distributions. $M$ is the distance matrix indicating the cost to transport one unit mass of probability between two discrete distributions. Its structure will be specified in Section~\ref{WOOD_Properties}. For an OOD sample $x_{\text{OOD}}\in\mathcal{X}_\text{OOD}$, we would like the distance $W(f_{\theta}(x_{\text{OOD}}),y^k)$ to be large for each class $k\in\mathcal{K}_\text{InD}$ , which is equivalent to enlarge the distance to its closest class. Thus, for a random sample $x$, the proposed Wasserstein-distance-based score function is defined as: 
\begin{align}
    s(x) &= \min_{k}W(f_{\theta}(x),y^{k})
    \notag\\
    &= \min_{k}\inf_{P}\langle P, M\rangle,
    \notag\\
    \text{s.t.}\;\;
    P\mathbf{1}_{K} &= y^{k},
    \notag\\
    P^{\top}\mathbf{1}_{K} &= f_{\theta}(x),
    \notag\\
    k &\in \mathcal{K}_\text{InD}.
    \label{Eq3-6}
\end{align}

Note that the above Wasserstein-distance-based score function does not require to assume that the OOD samples belong to the same class. Instead, it reveals the difference between InD and OOD samples by evaluating the closest distance from an arbitrary sample to any InD classes. Ideally, if $x$ is an InD sample, $s(x)$ should be close to 0. If $x$ is an OOD sample, $s(x)$ should be close to a specific positive value which is determined by the choice of the distance matrix $M$. To this end, we impose a threshold parameter $\epsilon$ on the defined Wasserstein-distance-based score function and construct the OOD detector $g$ as:
\begin{align}
    g(x) = \left\{
                          \begin{aligned}
                          1, \;\;  & s(x) > \epsilon\\
                          0, \;\;  & s(x) \leq \epsilon,
                          \end{aligned}
                          \right.
    \label{Eq3-7}
\end{align}
The training of the threshold parameter $\epsilon$ will be discussed in Section \ref{hyperparameters}.

\subsection{WOOD Loss Function}
\label{sec:WOOD Loss}
With the specifically designed score function, the instant question is how to train the classifier to separate InD and OOD samples in the space defined by the score function as well as preserving the classification performance. As a key component for training the classifier, a new loss function is designed for OOD detection by mimicking the human intuitive logic, that is, to regard a sample as OOD if it does not belong to any existing classes based on human's experience/memory. The Wasserstein-distance-based score function is incorporated in the loss, and the classifier is trained to assign the InD samples with their correct labels and keep all the OOD samples away from any InD classes. Motivated by this logic, we define the WOOD loss function as
\begin{align}
    \mathcal{L}(\mathcal{D}_\text{InD}, \mathcal{D}_\text{OOD}) &= \frac{1}{N_\text{InD}}\sum_{\left({x}_\text{InD}, k\right) \in \mathcal{D}_\text{InD}}-\log \left(f_{{\theta}}({x}_\text{InD})^{\top}y^{k}\right) 
    \notag\\
    &-\beta \frac{1}{N_\text{OOD}}\sum_{{x}_\text{OOD}\in \mathcal{D}_\text{OOD}}\min_{k}\inf_{P}\langle P, M\rangle,
    \notag\\
    \text{s.t.}\;\;
    P\mathbf{1}_{K} &= y^{k},
    \notag\\
    P^{\top}\mathbf{1}_{K} &= f_{{\theta}}({x}_\text{OOD}
    ),
    \notag\\
    k &\in \mathcal{K}_\text{InD},
    \label{Eq3-5}
\end{align}
where $N_{\text{InD}}$, $N_{\text{OOD}}$ are the sample sizes of $\mathcal{D}_\text{InD}, \mathcal{D}_\text{OOD}$, respectively; $\beta$ is the hyperparameter to balance the loss in the classification error and OOD detection error. The WOOD loss includes two terms. The first term on the right hand side of Equation~(\ref{Eq3-5}) is the cross-entropy loss that aims to correctly match the InD samples with the corresponding labels. The second term on the right hand side of Equation~(\ref{Eq3-5}) is the Wasserstein-distance-based score, which calculates the minimum distances between the predicted softmax score vector and the one-hot vector of any InD classes. By minimizing the WOOD loss, the classifier will tend to assign correct labels to InD samples and keep the predicted results of OOD samples away from any of the InD classes.

\subsection{Properties of WOOD with Different Distance Matrices}
\label{WOOD_Properties}
The properties of WOOD are influenced by the choices of distance matrices $M$ in Equation~(\ref{Eq3-5}). Here we provide two commonly used distance matrices - the binary distance matrix and dynamic distance matrix, and discuss their influence on the properties of WOOD.

\subsubsection{Binary Distance Matrix}
The formulation of the binary distance matrix is given in Equation (\ref{Eq3-8}). Recall that the distance matrix in the Wasserstein distance is to evaluate the unit cost in transporting the probability mass. In the binary distance matrix $M_\text{Bi}$, the diagonal entries are set as $0$, while the other entries are set as 1. Given that the binary distance matrix treats all the classes equally, transporting the unit probability mass to any different class will consistently cause a cost of 1.

\begin{align}
    M_\text{Bi} = \begin{bmatrix} 
    0 & 1 & \dots \\
    \vdots & \ddots & \\
    1 & & 0 
    \end{bmatrix}
\label{Eq3-8}
\end{align}

The binary distance matrix is widely used in applications of Wasserstein distance, such as classification \cite{NIPS2015_a9eb8122}. Our work is the first to explore its performance in OOD detection. The advantages of binary distance matrix include: (i) the formulation is straightforward and easy to understand; (ii) it is compatible with various classifier structures and diverse applications; (iii) it is a metric matrix so that the Wasserstein distance satisfies the axioms of a distance (shown in Equation (\ref{Eq2-2})).

\subsubsection{Dynamic Distance Matrix}
\label{sec:3-4-2}
Different from the binary distance matrix, we tailor a dynamic distance matrix for the OOD detection purpose. This distance matrix is calculated by the pair-wise absolute difference between elements in a predicted softmax score $f_{\theta}(x)$ and an one-hot vector label $y^{k}$, which is given as

\begin{align}
    M_{\text{Dy}}&(f_{\theta}(x), y^{k}) 
    \notag\\
    &= \left[ f_{\theta}(x) - \mathbf{0}_{K},...,\underbrace{\mathbf{1}_{K} - f_{\theta}(x)}_{k^{th}},...,f_{\theta}(x) - \mathbf{0}_{K}\right],
\label{Eq3-9}
\end{align}
where $\mathbf{0}_{K} \in \mathbb{R}^{K}$ denotes the $K$-dimensional vector whose elements are all zeros. In Equation (\ref{Eq3-9}), the dynamic distance matrix is denoted as a function of $f_{\theta}(x)$ and $y^{k}$, which will be dynamically changed with different marginal distributions. In the following parts, we use $M_{\text{Dy}}$ for simplicity. The dynamic distance matrix is not a metric matrix, and hence the Wasserstein distance does not satisfy the axioms of distance. However, the properties of the Wasserstein distance with a dynamic distance matrix make it a good fit for the OOD detection problem. First, the dynamic distance matrix $M_{\text{Dy}}$ is invariant with respect to any one-hot vector $y^{k}$, setting an equal penalty on the loss function when an OOD sample is close to any InD classes. The proposition is formally expressed as follows.
\begin{proposition}
    \begin{equation}
        W_{M_{\text{Dy}}}(f_{\theta}(x),y^{1}) = ,...,=W_{M_{\text{Dy}}}(f_{\theta}(x),y^{K})
    \end{equation}
  \label{Propert-1}
\end{proposition}
This proposition reveals the intuition of designing the dynamic distance matrix. Unlike the classification task, the OOD detection mainly cares about whether an incoming sample is OOD or not, which means the proposed method can regard all the InD classes as the same. Using the Wasserstein distance with a dynamic distance matrix, the minimal operator in the score function and loss function in Equations~(\ref{Eq3-6}) and (\ref{Eq3-5}) can be eliminated, thus the formulation is  simplified as

\begin{align}
    s(x) = &\inf_{P}\langle P, M\rangle,
    \label{Eq3-10}
\end{align}
\begin{align}
    \mathcal{L}(\mathcal{D}_\text{InD}, \mathcal{D}_\text{OOD}) &= \frac{1}{N_\text{InD}}\sum_{\left(x_{\text{InD}}, k\right) \in \mathcal{D}_\text{InD}}-\log \left(f_{\theta}(x_{\text{InD}})^{\top}y^{k}\right)
    \notag\\
    &-\beta \frac{1}{N_\text{OOD}}\sum_{x_{\text{OOD}}\in \mathcal{D}_\text{OOD}}\inf_{P}\langle P, M\rangle,
    \label{Eq3-11}
\end{align}
where $k \in \mathcal{K}_\text{InD}$ is randomly selected because different class labels have no influence in the Wasserstein distance using the dynamic distance matrix.

The second proposition described below shows the maximal point of the Wasserstein-distance-based score function.
\begin{proposition}
    Consider $W_{M_{\text{Dy}}}(f_{\theta}(x),y^{k})$ as a function of $f_{\theta}(x)$, it reaches the maximum when $f_{\theta}(x) = (\frac{1}{K}, ..., \frac{1}{K})$.
  \label{Propert-2}
\end{proposition}
This proposition implies that the score function reaches maximum when the classifier cannot determine which class the sample falls into. The combination of the above two propositions capture the geometric shape of $W_{M_{\text{Dy}}}(f_{\theta}(x),y^{k})$ as a function of the softmax vector $f_{\theta}(x)$: i.e. a function defined on a hyperplane, taking the maximal value at the center $(\frac{1}{K}, ..., \frac{1}{K})$, and taking the minimal value on the boundary $y^{k},k\in\mathcal{K_{\text{InD}}}$. This geometric feature lays a foundation for incorporating $W_{M_{\text{Dy}}}(f_{\theta}(x),y^{k})$ into the loss function and using it as the score function. The loss function will train the classifier to map InD and OOD samples in opposite directions. For the InD samples, the model is trained to output softmax scores close to their labels, while for the OOD samples, the model is trained to output softmax scores that are away from any labels and close to $(\frac{1}{K}, ..., \frac{1}{K})$. The proof of Propositions \ref{Propert-1} and \ref{Propert-2} can be found in Appendix \ref{Appendix-A}.

\subsubsection{Influence of Different Distance Matrices on Computational Complexity}
\label{Complexity}

The Equations (\ref{Eq3-10}) and (\ref{Eq3-11}) indicate that compared with binary distance matrix, using dynamic distance matrix reduces the computational cost in calculating the WOOD loss and the score function, because the minimal operator is eliminated. The score function $s(x)$ is used as an example to theoretically illustrate the influence on computational complexity. Suppose there is a random input $x$, if the dynamic distance matrix is used, the computational complexity of calculating the regularized Wasserstein distance between $f_{\theta}(x)$ and $y^{k}$ is $O(K^{2})$ \cite{NIPS2013_af21d0c9} (introduced in Appendix \ref{Appendix-D}). Thus, the computational complexity of $s(x)$ with dynamic distance matrix (Equation (\ref{Eq3-10})) is $O(K^{2})$. In contrast, if the binary distance matrix is used, as shown in Equation (\ref{Eq3-6}), the score function needs to find the minimum Wasserstein distance between the output softmax score and all $K$ possible InD classes, which increases the computational complexity to $O(K^{3})$. A similar influence also exists when calculating WOOD loss with different formulations of the distance matrix. From the perspective of computational complexity, the dynamic distance matrix is preferably selected.

\section{Computational and Theoretical Perspectives}
\label{learning}
Given the formulation of WOOD, the following task is to learn the model parameters with WOOD loss. In this section, we will at first show the procedure of model training, and then discuss the statistical property of the proposed method. 

\begin{figure*}[!ht]
    \includegraphics[width=0.9\linewidth]{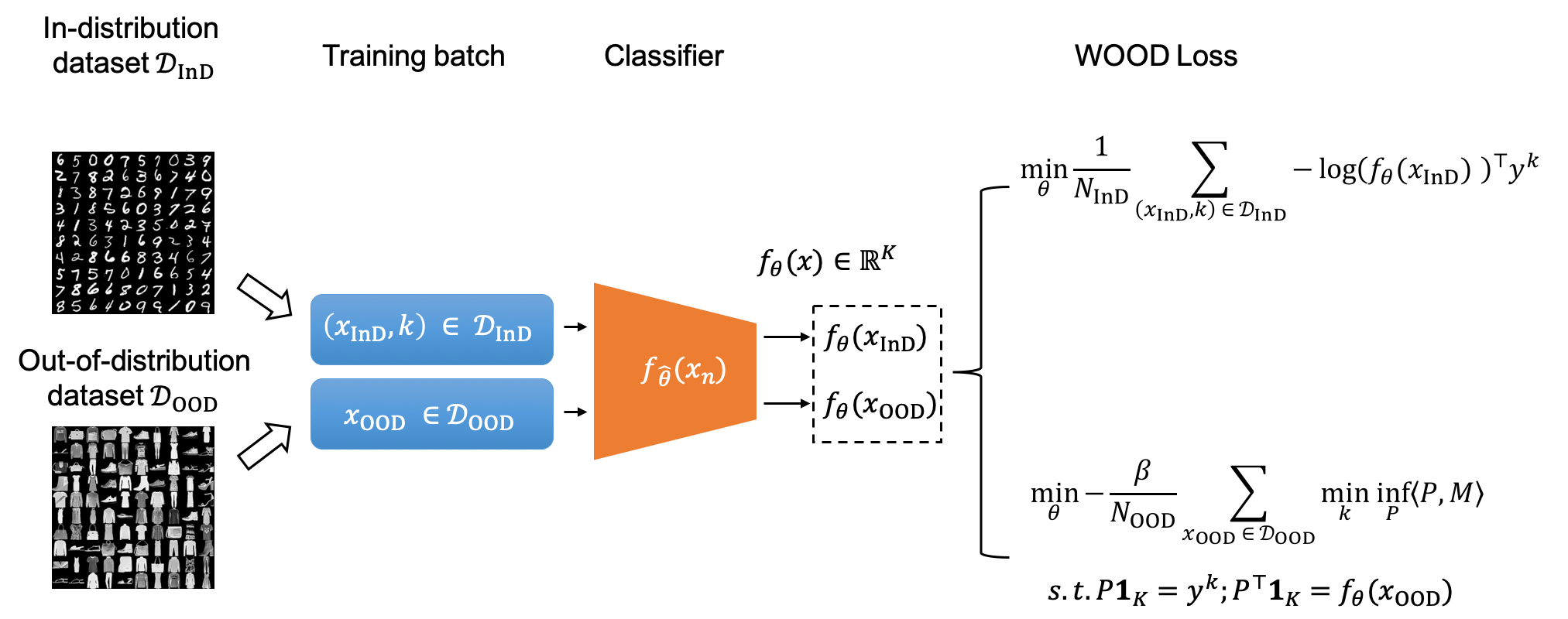}
    \centering 
    \caption{Overview of adapting WOOD loss to train image classifier}
    \label{fig:framework_overview}
\end{figure*}

\subsection{Gradient-based Model Training}
\label{gradient}
The gradient-based method is usually applied in deep learning to learn the model parameters. As shown in Equation (\ref{Eq3-5}), the proposed WOOD loss consists of two components, the cross-entropy used for classifying InD samples and the Wasserstein-based score function for detecting OOD samples. The gradient of cross-entropy has been studied in the classification models \cite{https://doi.org/10.1111/j.2517-6161.1958.tb00292.x}. Here we focus on deriving the gradient of the Wasserstein-based score function. To do this, it is suffice to find a way to compute
\begin{equation}
    \nabla_\theta\,\min_{k}\,\inf_{P}\,\langle P, M\rangle.\nonumber
\end{equation}

Under the following proposition (see proof in Appendix \ref{Appendix-C}), the partial derivative and minimal operators can be interchanged. 
\begin{proposition}
    Suppose $\eta_i(\cdot)$'s are mutually different at $\theta$, and $\eta_i(\theta)$ is first-order differentiable, with $i^{*}=arg\,\min_i \eta_i(\theta)$, we have $\nabla_\theta \min_i \eta_i(\theta) = \nabla_\theta \eta_{i^*}(\theta)$.
    \label{prop:4-1}
\end{proposition}

That is to say:
\begin{equation}
    \nabla_\theta\,\min_{k}\,\inf_{P}\,\langle P, M\rangle=\min_{k}\,\nabla_\theta\,\inf_{P}\,\langle P, M\rangle.
\end{equation}

Note that $P$ and $M$ depend on $k^*$ that minimizes the Wasserstein distance from the one-hot vector in class $k$ to $f_\theta(x)$. According to the chain rules,
\begin{equation}
    \nabla_\theta\inf_{P \in \Pi^{*}}\,\langle P, M\rangle=
    \nabla_{f_\theta(x)}\inf_{P \in \Pi^{*}}\,\langle P, M\rangle \times \nabla_\theta\,f_\theta(x),
    \label{eq:multiply}
\end{equation}
where $\Pi^{*}$ is the set of joint distributions with $f_{\theta}(x)$ and $y^{k^{*}}$ as marginal distributions. Note that $\nabla_\theta\, f_\theta(x)$ is the gradient of the softmax output of a deep neural network, which has been widely studied in the related literature. Therefore we only need to focus on the gradient term $\nabla_{f_\theta(x)}\inf_{P \in \Pi^{*}}\,\langle P, M\rangle$, which is equivalent to the gradient of  $W\left(f_\theta(x),y^{k^*}\right)$ in Equation~(\ref{Eq3-1}). 

We derive the gradient of Wasserstein distance with respect to $f_{\theta}(x)$ following similar procedures in \cite{NIPS2015_a9eb8122}. The basic idea is that because Wasserstein distance (shown in Equation (\ref{Eq3-1})) is a linear programming (LP), its gradient with respect to $f_{\theta}(x)$ can be computed via Lagrangian duality. The optimal values of primal and dual problems are equal. So the value of the corresponding dual variable at the optimal point is the desired gradient of the Wasserstein distance with respect to $f_{\theta}(x)$. Suppose we have dual variables $a,b$ corresponding to primal variables $y^{k}, f_{\theta}(x)$, respectively, and two intermediate variables $u,v$ are defined as $u = \exp^{-\lambda a - \frac{1}{2}}, v = \exp^{-\lambda b - \frac{1}{2}}$, the $u$ and $v$ can be solved iteratively by Sinkhorn-Knopp algorithm \cite{NIPS2013_af21d0c9}, which is given in Algorithm \ref{Alg:Sinkhorn} (in Appendix \ref{Appendix-E}). Given the value of $v$ when converging as $v^{*}$, the gradient of Wasserstein distance with respect to $f_{\theta}(x)$ is given as

\begin{align}
    \nabla_{f_{\theta}(x)}\,W(f_{\theta}(x), y^{k}) = b^{*} = -\frac{1}{\lambda}(\log v^{*} + \frac{1}{2}).
    \label{Eq4-7}
\end{align}

The detailed derivation of gradient is in Appendix \ref{Appendix-E}. To this point, we can compute the gradient of WOOD loss with respect to the softmax scores of InD samples $x_{\text{InD}}$ and OOD samples $x_{\text{OOD}}$ as 
\begin{align}
    \nabla_{f_{\theta}(x_{\text{InD}})} \mathcal{L} &= \frac{-1}{N_{\text{InD}}}\left(\frac{1}{f_{\theta}(x_{\text{InD}})}\right)^{\top}y^{k},
    \label{Eq4-8}
\end{align}

\begin{align}
    \nabla_{f_{\theta}(x_{\text{OOD}})}\mathcal{L} &= \frac{\beta}{\lambda N_{\text{OOD}}}(\log v^{*} + \frac{1}{2}),
    \label{Eq4-9}
\end{align}
where $v^{*}$ is calculated by Algorithm \ref{Alg:Sinkhorn} with $y^{k^{*}}$ and $f_{\theta}(x_{\text{OOD}})$ as the inputs. The gradient of WOOD loss with respect to the model parameters $\theta$ can be further derived via the chain rule.

\subsection{Training Algorithm}
The training process using WOOD loss is summarized in Algorithm \ref{Alg:Training}. $\mathbf{X}$ and $\mathbf{K}$ denote one batch of samples and their corresponding labels, respectively. In each batch, $\mathbf{X}_{\text{InD}}$ and $Y_{\text{InD}}$ denote the portion of InD images and their labels, respectively. $\mathbf{X}_{\text{OOD}}$ denotes the portion of OOD images in each batch. Labels of OOD images $Y_{\text{OOD}}$ are only used to identify the OOD samples during training. Each image is denoted as a tensor with the dimension of $C\times H_{1}\times H_{2}$, in which $C$ represents the number of channels, $H_{1}$ and $H_{2}$ represent the length and width, respectively. The number of InD and OOD samples in each batch is $B_{\text{InD}}$ and $B_{\text{OOD}}$, respectively. The algorithm mainly consists of four steps: (i) in each iteration, a batch of InD samples is mixed with fewer OOD samples, and the label $\mathbf{1}_{K}$ assigned to OOD samples is only used to identify them in the training phase and not used in calculation; (ii) in the forward propagation, the InD and OOD samples are fed into the classifier to generate the softmax scores, and the value of the WOOD loss is calculated by Equation (\ref{Eq3-5}); (iii) in the backward propagation, the gradients of the WOOD loss with respect to the softmax scores are calculated by Equations (\ref{Eq4-8}, \ref{Eq4-9}), and the gradients with respect to model parameters are generated via chain rule; (iv) the model parameters are then updated via the gradient-based method.

\begin{algorithm}[ht]
  \caption{Training algorithm of WOOD}
  \begin{algorithmic}[1]
    \Inputs{$\mathcal{D}_{\text{InD}}, \mathcal{D}_{\text{OOD}}$}
    \Initialize{$\beta$\Comment{hyperparameter in loss function}\\
    $K$\Comment{the number of class}\\
    \text{ClsModel}\Comment{classification model}}
    \State{\textbf{Training in one epoch:}}
    \For{$\mathbf{X}_{\text{InD}} \in \mathbb{R}^{B_{\text{InD}}\times C\times H_{1}\times H_{2}}, Y_{\text{InD}} \in \mathbb{R}^{B_{\text{InD}}\times K}$ in $\mathcal{D}_{\text{InD}}$}
      \State Randomly select $\mathbf{X}_{\text{OOD}} \in \mathbb{R}^{B_{\text{OOD}}\times C\times H_{1}\times H_{2}}$ in $\mathcal{D}_{\text{OOD}}$
      \State{\textbf{Forward Propagation:}}
      \State $Y_{\text{OOD}} \in \mathbb{R}^{B_{\text{OOD}}\times K} \gets \mathbf{1}_{K}$ \Comment{indicate OOD samples} 
      \State $\mathbf{X} \gets [\mathbf{X}_{\text{InD}}, \mathbf{X}_{\text{OOD}}]\in \mathbb{R}^{(B_{\text{InD}}+B_{\text{OOD}})\times C\times H_{1}\times H_{2}}$
      \State $Y \gets [Y_{\text{InD}}, Y_{\text{OOD}}] \in \mathbb{R}^{(B_{\text{InD}}+B_{\text{OOD}})\times K}$
      \State $f_{\theta}(\mathbf{X}) \gets \text{softmax}\left(\text{ClsModel}(\mathbf{X})\right)$ 
      \State{\textbf{The value of Loss:}}
      \State $(\mathbf{X}_{\text{InD}}, f_{\theta}(\mathbf{X}_{\text{InD}})) \gets (\mathbf{X}, f_{\theta}(\mathbf{X}))[\text{InD-idx}]$
      \State $(\mathbf{X}_{\text{OOD}}, f_{\theta}(\mathbf{X}_{\text{OOD}})) \gets (\mathbf{X},f_{\theta}(\mathbf{X}))[\text{OOD-idx}]$
      \State $\text{InD-Loss} \gets \frac{1}{B_{\text{InD}}}\sum -\log \left(f_{\theta}(x_{\text{InD}})^{\top}y^{k}\right) $
      \State $\text{OOD-Loss} \gets \frac{1}{B_{\text{OOD}}}\sum\min_{k}\inf\langle P, M\rangle$
      \State $\mathcal{L} \gets \text{InD-Loss} - \beta \times \text{OOD-Loss}$
      \State{\textbf{Backward Propagation:}}
      \State $\frac{\partial \mathcal{L}}{\partial f_{\theta}(x_{\text{InD}})} \gets -\frac{1}{B_{\text{InD}}}\times\left(\frac{1}{f_{\theta}(x_{\text{InD}})}\right)^{\top}y^{k}$
      \State $\frac{\partial \mathcal{L}}{\partial f_{\theta}(x_{\text{OOD}})} \gets \frac{\beta}{\lambda B_{\text{OOD}}}\times(\log v^{*} + \frac{1}{2})$\Comment{detailed derivations in Section \ref{gradient}}
      \State $\frac{\partial \mathcal{L}}{\partial \theta} \gets \frac{\partial \mathcal{L}}{\partial f_{\theta}(x)}\frac{\partial f_{\theta}(x)}{\partial \theta}$
      \State Update model parameters with gradients
    \EndFor
  \end{algorithmic}
  \label{Alg:Training}
\end{algorithm}

\subsection{Statistical Properties of the WOOD Loss}
In this section, we investigate the statistical properties of the WOOD loss. Suppose $(x_{\text{InD}}, k) \in \mathcal{D}_{\text{InD}}$ and $x_\text{OOD} \in \mathcal{D}_{\text{OOD}}$ are independent samples, $N_{\text{InD}}$ and $N_{\text{OOD}}$ are the number of samples in $\mathcal{D}_{\text{InD}}$ and $\mathcal{D}_{\text{OOD}}$, respectively, $f_{\hat{\theta}}$ is the empirical risk minimizer, and $\mathcal{F}$ is the hypothesis space of functions mapping into $\mathbb{R}^{K}$, we have the statistical learning bound of the WOOD loss given in Theorem \ref{them4-3}.

\begin{theorem}
For any $\delta > 0$, with probability at least $1-\delta$, it holds that
\begin{align}
    &\mathbb{E}[\mathcal{L}(\mathcal{D}_\text{InD}, \mathcal{D}_\text{OOD})] - \inf_{f_{\theta}\in\mathcal{F}} \mathbb{E}[\mathcal{L}(\mathcal{D}_\text{InD}, \mathcal{D}_\text{OOD})]
    \notag\\
    \leq & \frac{4\sqrt{2}}{m}\mathcal{R}_{N_{\text{InD}}}(\mathcal{F}) + 2\sqrt{\frac{(1-m)^{2}\log (1/\delta)}{2N_{\text{InD}}m^{2}}} 
    \notag\\
    &+ \beta\left(16\alpha_{M}\mathcal{R}_{N_{\text{OOD}}}(\mathcal{F}) + 2\sqrt{\frac{\alpha_{M}^{2}\log (1/\delta)}{2N_{\text{OOD}}}}\right),
\end{align}
\label{them4-3}
where $\alpha_{M}$ is the maximum entry in all matrices $M$, $m$ is the minimum value in all predicted softmax scores $f_{\hat{\theta}}(x_\text{InD})$, $\mathcal{R}_{N_{\text{InD}}}$ and $\mathcal{R}_{N_{\text{OOD}}}$  are Rademacher complexity \cite{10.5555/944919.944944}.
\end{theorem}

The detailed derivation is provided in Appendix \ref{Appendix-B}. The upper bound in Theorem \ref{them4-3} guarantees that with a sufficient number of training samples, the difference between the minimized loss value and the global optimum is bounded by an arbitrarily small value, which indicates the loss value achieved by the empirical minimizer will approach the global optimum.

\section{Experiments}
\label{experiment}
In this section, we introduce the experiment setup and demonstrate the performance of WOOD on different InD and OOD datasets. 

\subsection{Experiment Setup}
\subsubsection{Datasets}
The datasets used in the experiments are MNIST \cite{lecun-mnisthandwrittendigit-2010}, FashionMNIST \cite{xiao2017}, CIFAR-10 \cite{krizhevsky2009learning}, SVHN \cite{37648}, and downsampled TinyImageNet. TinyImageNet is a subset of ImageNet \cite{5206848} which contains images with the shape of $3 \times 64 \times 64$ in 200 classes. To keep the image dimension of the TinyImageNet consistent with CIFAR-10 and SVHN, we further downsample it into the dimension of $3 \times 32 \times 32$ by resizing (TinyImageNet-r) or cropping (TinyImageNet-c). The basic information of these datasets, including the number of classes, image dimension, the number of training samples, and the number of testing samples are summarized in Table \ref{table5-1}. These datasets are used as InD and OOD samples to validate the performance of WOOD, which is indicated in the first column of Table \ref{table5-2}.

\begin{table*}[ht]
    \centering
    \caption{Basic Information of Datasets}
    \begin{tabular}{ccccc}
    \specialrule{.1em}{.05em}{.05em}
    \\[-.7em]
    Datasets & Number of Class & Image Dimension & Number of Training & Number of Testing \\
    \specialrule{.1em}{.05em}{.05em}
    \\[-.7em]
    MNIST & 10 & $28\times 28$ & 60000 & 10000 \\
    FashionMNIST & 10 & $28\times 28$ & 60000 & 10000 \\
    CIFAR-10 & 10 & $3\times 32\times 32$ & 50000 & 10000 \\
    SVHN & 10 & $3\times 32\times 32$ & 73257 & 26032 \\
    TinyImageNet-r & 200 & $3\times 32\times 32$ & 100000 & 10000\\
    TinyImageNet-c & 200 & $3\times 32\times 32$ & 100000 & 10000\\
    \\[-.7em]
    \specialrule{.1em}{.05em}{.05em}
    \end{tabular}
    \label{table5-1}
\end{table*}

\subsubsection{Hyperparameters}
\label{hyperparameters}
Considering the proposed WOOD is a general framework to enable classifiers to detect OOD samples, we test its performance on the state-of-the-art classifier DenseNet \cite{8099726} in the experiments. For the DenseNet, we follow the setup introduced in \cite{8099726}, with model depth $100$, growth rate $12$, and dropout rate $0$. 

$\beta$ is the hyperparameter used in WOOD loss to balance the focus of the classifier in classifying the InD samples and detecting the OOD samples. The optimal $\beta = 0.1$ is obtained in the experiments that enable the classifier to receive a good performance in OOD detection without decaying its performance in classification. The value of $\beta$ is determined by the grid-search with the range of $[0,1]$ and the step size as $0.1$. $B_{\text{InD}}$ and $B_{\text{OOD}}$ in Algorithm \ref{Alg:Training} are the batch sizes of InD and OOD samples in each training iteration. We set $B_{\text{InD}} = 50$ and $B_{\text{OOD}} = 10$ in training. $\epsilon$ in Equation (\ref{Eq3-7}) is determined by the $95\%$ TNR on InD testing samples.

\subsubsection{Baseline Methods and Evaluation Metrics}
To demonstrate the performance of the WOOD, two state-of-the-art OOD detection methods ODIN \cite{liang2020enhancing} and Maha \cite{lee2018simple} are selected as the baseline methods.

All the OOD detection methods are compared by two evaluation metrics: (1) The FNR of OOD samples at $95\%$ TNR, which indicates how many OOD samples are misidentified when the threshold $\epsilon$ is set to ensure $95\%$ of InD samples are correctly identified. The FNR at $95\%$ TNR is the lower the better. (2) The area under the receiver operating characteristic curve (AUROC), which is the higher the better.

\subsection{Results}
\begin{table*}[!ht]
    \centering
    \caption{Experiment Results of the Proposed WOOD Framework and Baseline Methods}
    \begin{tabular}{cccccc}
    \specialrule{.1em}{.05em}{.05em}
    \\[-.7em]
    Datasets & \diagbox{Metrics}{Methods} & ODIN & Maha & \makecell{WOOD (proposed) \\ binary distance matrix}  & \makecell{WOOD (proposed) \\ dynamic distance matrix} \\
    \\[-.7em]
    \specialrule{.1em}{.05em}{.05em}
    \\[-.7em]
    \multirow{2}{*}{\makecell{InD: MNIST \\ OOD: FashionMNIST}} & FNR ($95\%$  TNR) & 0.8754& $<0.000001$& $<0.000001$& $<0.000001$\\
    & AUROC & 0.741& $>0.999999$& $>0.999999$& $>0.999999$ \\
    \\[-.7em]
    \hline
    \\[-.7em]
    \multirow{2}{*}{\makecell{InD: FashionMNIST \\ OOD: MNIST}} & FNR ($95\%$  TNR) & 0.7809& 0.1653& $<0.000001$& $<0.000001$ \\
    & AUROC & 0.891& 0.969& $>0.999999$& $>0.999999$ \\
    \\[-.7em]
    \hline
    \\[-.7em]
    \multirow{2}{*}{\makecell{InD: CIFAR-10 \\ OOD: SVHN}} & FNR ($95\%$  TNR)& 0.1591& 0.1870& 0.0005 & 0.0046 \\
    & AUROC& 0.962 & 0.934 & 0.999 & 0.998\\
    \\[-.7em]
    \hline
    \\[-.7em]
    \multirow{2}{*}{\makecell{InD: CIFAR-10 \\ OOD: TinyImageNet-r}} & FNR ($95\%$  TNR)& 0.0430& 0.1724& 0.0095 & 0.0069 \\
    & AUROC& 0.991 & 0.934 & 0.988 & 0.993\\
    \\[-.7em]
    \hline
    \\[-.7em]
    \multirow{2}{*}{\makecell{InD: CIFAR-10 \\ OOD: TinyImageNet-c}} & FNR ($95\%$  TNR)& 0.1340& 0.4878& $<0.000001$& 0.0003\\
    & AUROC& 0.975 & 0.898 & $>0.999999$ & 0.999 \\
    \\[-.7em]
    \specialrule{.1em}{.05em}{.05em}
    \end{tabular}
    \label{table5-2}
\end{table*}
The performances of the proposed WOOD and baseline methods are compared on five different combinations of InD and OOD datasets. The experiment results are summarized in Table \ref{table5-2}. In general, the WOOD method receives comparable performance when using binary and dynamic distance matrices, and both of them outperform the baseline methods consistently. More specifically, both the ODIN and Maha methods show strength in detecting OOD samples from some datasets while failing in others. For example, ODIN performs well in detecting SVHN, TinyImageNet-r, and TinyImageNet-c from CIFAR-10. However, its performance decays in distinguishing MNIST and FashionMNIST from each other. Maha receives good performance in most cases except for detecting TinyImageNet-c from CIFAR-10. Compared with ODIN and Maha, the proposed WOOD method improves both the FNR at $95\%$ TNR and AUROC in all the cases and receives consistent outstanding performance in identifying the OOD samples.

\begin{figure}[!ht]
    \includegraphics[width=\linewidth]{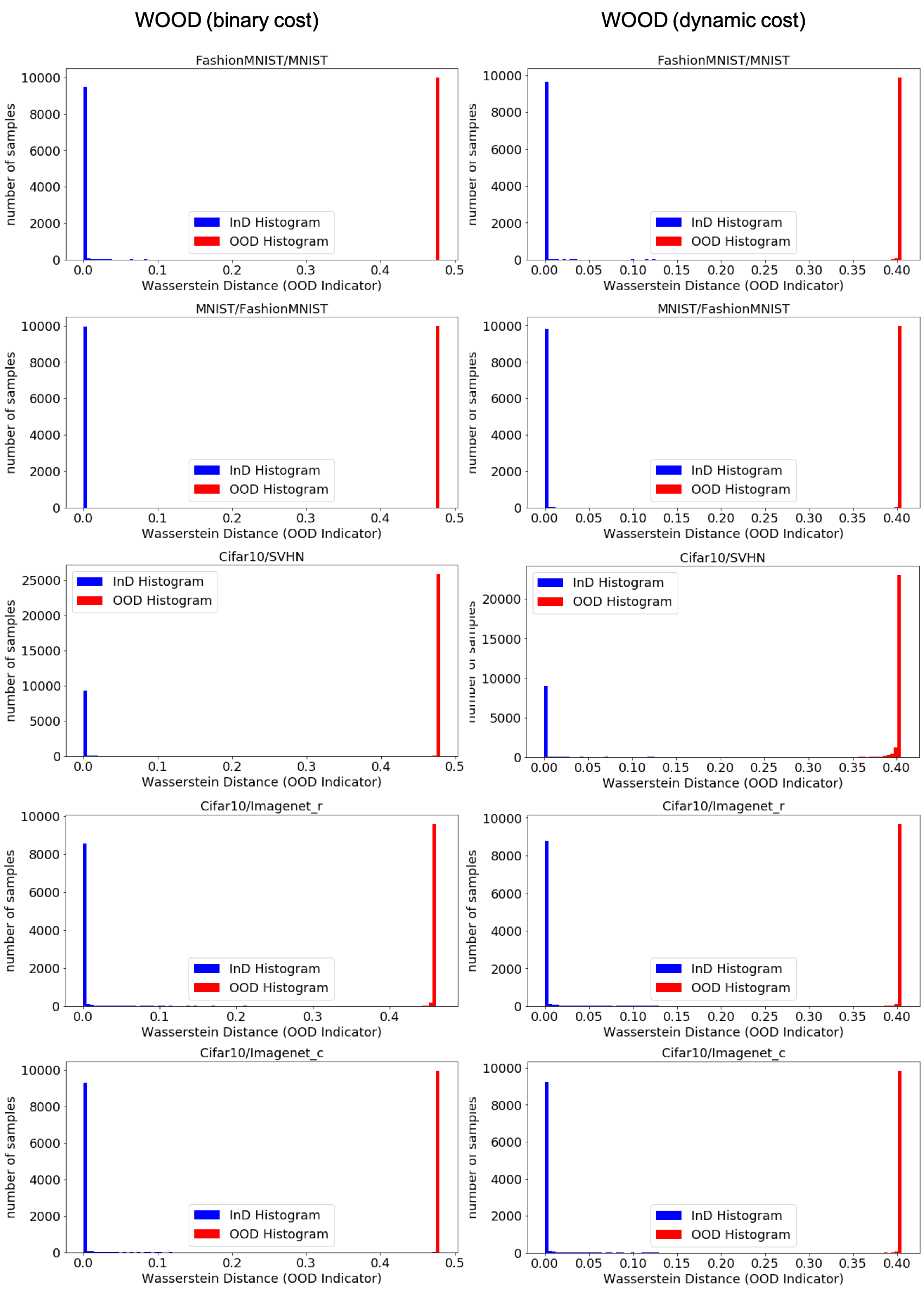}
    \centering 
    \caption{Histogram of OOD Indicator in WOOD}
    \label{fig:hist_viz}
\end{figure}

\begin{figure}[!ht]
    \includegraphics[width=0.9\linewidth]{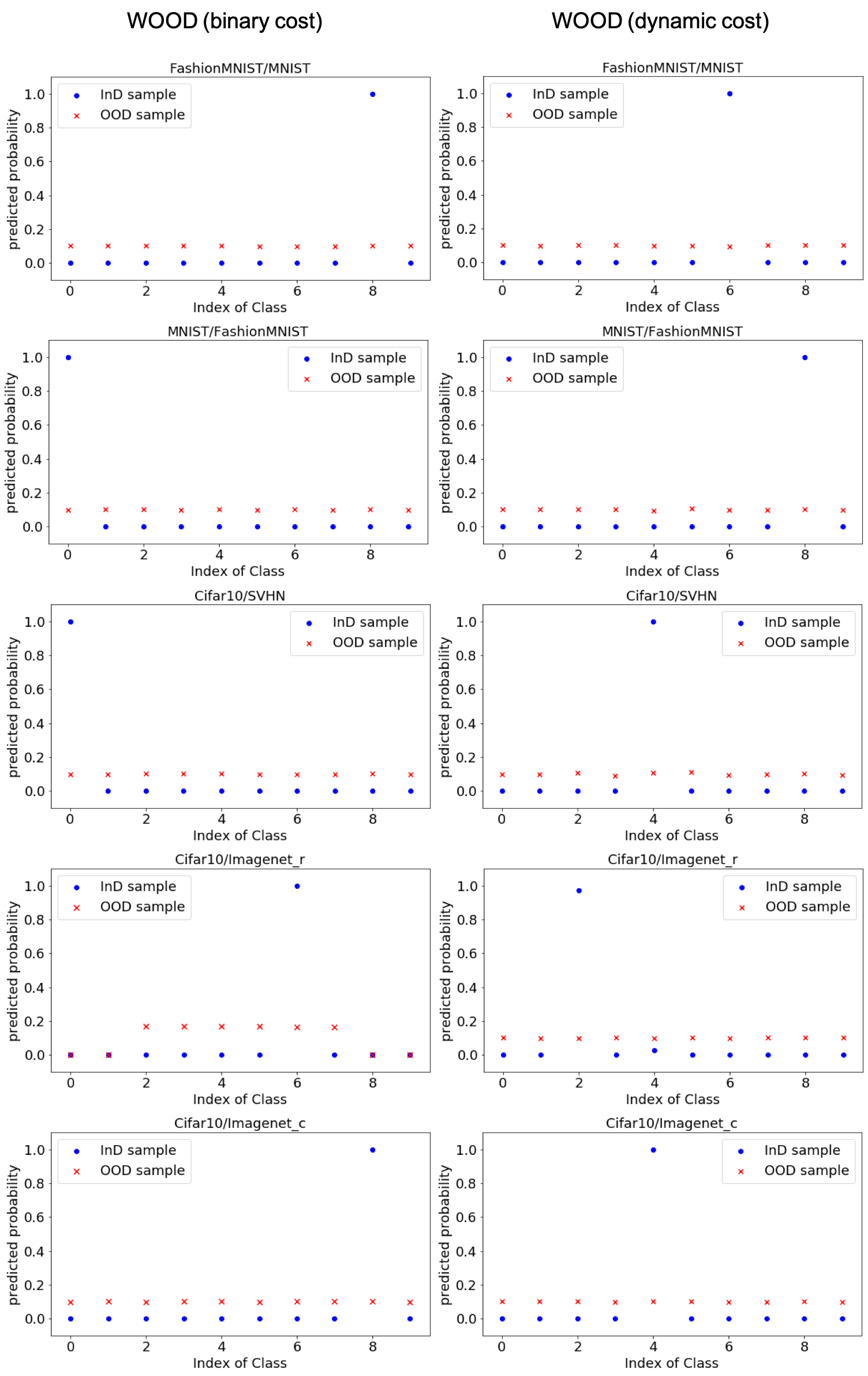}
    \centering 
    \caption{Predicted Probability of Randomly Selected InD and OOD Sample}
    \label{fig:sample_pred_viz}
\end{figure}

The score function in the WOOD is used to distinguish InD and OOD samples. Ideally, after training the classifier with WOOD loss, the values of the score function for the InD samples should be concentrated close to 0, while the values of the score function for the OOD samples should be concentrated away from 0. We demonstrate the histogram of the score function for different InD and OOD datasets in Fig. \ref{fig:hist_viz}. The title of each subplot represents the InD and OOD datasets, for example ``FashionMNIST/MNIST" represents the FashionMNIST is the InD dataset and MNIST is the OOD dataset. In Fig. \ref{fig:hist_viz}, blue bars represent the histogram of score function for InD samples and red bars represent the histogram of score function for OOD samples. We can clearly figure out that after training the classifier with WOOD loss, the score function can successfully distinguish the InD and OOD datasets. This also explains why the WOOD framework consistently receives great performance in identifying all the OOD datasets.

As we discussed in Section \ref{sec:3-4-2}, when using the dynamic distance matrix in WOOD loss, for the OOD samples, the classifier is trained to output predicted probability away from any labels and close to $(\frac{1}{K},...,\frac{1}{K})$. To validate this property, we randomly select one OOD sample and one InD sample and generate the predicted probabilities from the trained classifier, which is shown in Fig. \ref{fig:sample_pred_viz}. We can find out that when using WOOD loss with dynamic distance matrix, the predicted probability of the InD sample is close to its label (the score function close to 0) and the predicted probability of the OOD sample is close to $(\frac{1}{10},...,\frac{1}{10})$ (the score function close to the maximum). This property also indicates that the value of the WOOD loss for the trained classifier converges to the minimum value, in which the first term (cross-entropy loss) approaches to $0$ and the second term (Wasserstein-based score function) approaches to the maximum value determined by the distance matrix. When using WOOD loss with binary distance matrix, the trained classifier has a similar property.

\subsection{Selection of Distance Matrix in WOOD}

From the results shown in Table \ref{table5-2}, we can conclude that there is no significant difference in the quantitative performance of WOOD with different distance matrices. As we discussed in Section \ref{Complexity}, theoretically, the dynamic distance matrix will reduce the computational complexity of Wasserstein-based score function from $O(K^{3})$ to $O(K^{2})$. We also validate this property in the experiment. To eliminate the influence from other computations, such as the calculation of cross-entropy loss, back-propagation, etc., we use the calculation of the Wasserstein-based score function for the single image in CIFAR-10 as an example. When using the dynamic distance matrix, the average computational time of the Wasserstein-based score function for 100 samples is $0.0035$ seconds. In comparison, when using the binary distance matrix, the average computational time of the Wasserstein-based score function for 100 samples is $0.0307$ seconds. Given $K = 10$ in CIFAR-10, the improvement of computational complexity is consistent with the theoretical analysis. Moreover, such improvement will be more significant when the number of InD classes $K$ further increases. Thus, if the Wasserstein distance is specifically used in OOD detection, we recommend using the dynamic distance matrix when $K$ is large.

It is worth noting that the Wasserstein distance with the binary distance matrix is not limited to be used in the OOD detection task and can be more generally used as the loss function in other computer vision tasks, such as object classification and detection \cite{NIPS2015_a9eb8122}. Thus, if the Wasserstein distance is used in OOD detection and other tasks simultaneously, we would recommend using the binary distance matrix to keep the distance measure consistent. This paper focuses on the OOD detection and tests both matrices in the experiments. Considering most of the InD datasets in our experiments have 10 classes, the difference in computational complexity is not significant when using these two matrices.

\subsection{Code Availability}
The implementation of this work is available at \url{https://github.com/wyn430/WOOD}. 

\section{Conclusion}
\label{conclusion}
OOD detection is a crucial task in protecting DNNs from cyber attacks. It can also improve the system's resilience and security. This paper proposes a Wasserstein-based out-of-distribution detection (WOOD) method to strengthen the classifiers with the ability to identify OOD samples. The performance of the proposed method is validated by different combinations of InD and OOD datasets and demonstrated outstanding results compared with other OOD detection methods. The advantages of the WOOD method can be summarized into five aspects: (1) It is generally compatible with various classifiers and has a little influence on the model complexity and training time; (2) The designed WOOD loss function for training the classifiers well follows the human's intuitive rationales and intelligence in identifying OOD samples, i.e., trying to keep OOD samples away from InD samples instead of trying to assign a specific label to them; (3) The Wasserstein distance fully exploits the dissimilarity between output discrete distributions for InD and OOD samples; (4) The specifically designed dynamic distance matrix reduces the computational complexity of WOOD loss and the score function; (5) The analysis of statistical learning bound provides the theoretical guarantees in training classifiers with the proposed WOOD loss function.

% Can use something like this to put references on a page
% by themselves when using endfloat and the captionsoff option.
\ifCLASSOPTIONcaptionsoff
  \newpage
\fi

% trigger a \newpage just before the given reference
% number - used to balance the columns on the last page
% adjust value as needed - may need to be readjusted if
% the document is modified later
%\IEEEtriggeratref{8}
% The "triggered" command can be changed if desired:
%\IEEEtriggercmd{\enlargethispage{-5in}}

% references section

% can use a bibliography generated by BibTeX as a .bbl file
% BibTeX documentation can be easily obtained at:
% http://mirror.ctan.org/biblio/bibtex/contrib/doc/
% The IEEEtran BibTeX style support page is at:
% http://www.michaelshell.org/tex/ieeetran/bibtex/
%\bibliographystyle{IEEEtran}
% argument is your BibTeX string definitions and bibliography database(s)
%\bibliography{IEEEabrv,../bib/paper}
%
% <OR> manually copy in the resultant .bbl file
% set second argument of \begin to the number of references
% (used to reserve space for the reference number labels box)
\bibliographystyle{IEEEtran}
\bibliography{IEEEabrv,references}
% \begin{thebibliography}{1}

% \bibitem{IEEEhowto:kopka}
% H.~Kopka and P.~W. Daly, \emph{A Guide to \LaTeX}, 3rd~ed.\hskip 1em plus
%   0.5em minus 0.4em\relax Harlow, England: Addison-Wesley, 1999.

% \end{thebibliography}

% biography section
% 
% If you have an EPS/PDF photo (graphicx package needed) extra braces are
% needed around the contents of the optional argument to biography to prevent
% the LaTeX parser from getting confused when it sees the complicated
% \includegraphics command within an optional argument. (You could create
% your own custom macro containing the \includegraphics command to make things
% simpler here.)
%\begin{IEEEbiography}[{\includegraphics[width=1in,height=1.25in,clip,keepaspectratio]{mshell}}]{Michael Shell}
% or if you just want to reserve a space for a photo:

% if have a single appendix:
%\appendix[Proof of the Zonklar Equations]
% or
%\appendix  % for no appendix heading
% do not use \section anymore after \appendix, only \section*
% is possibly needed

% use appendices with more than one appendix
% then use \section to start each appendix
% you must declare a \section before using any
% \subsection or using \label (\appendices by itself
% starts a section numbered zero.)
%

\pagebreak
\newpage

\appendices
\section{Properties of Wasserstein Distance between discrete distributions}
\label{Appendix-D}

It is worth noting that the distance matrix $M$ directly influences the properties of Wasserstein distance. Here we require the distance matrix $M$ to satisfy

\begin{equation}
    \begin{cases}
    M \in \mathbb{R}^{K \times K}_{+},\\
    M[t_{1},t_{2}] = 0 \text{ if and only if } t_{1}=t_{2}\leq K, \\
    M[t_{1},t_{2}]\leq M[t_{1},t_{3}]+M[t_{3},t_{2}], \text{ for any } t_{1},t_{2},t_{3}\leq K,
    \end{cases}
\label{Eq2-1}
\end{equation}
where $t_{1}, t_{2}, t_{3}$ denote the index of entries in matrix $M$. Under these conditions, the Wasserstein distance is a well-defined distance metric that satisfies the axioms of a distance

\begin{equation}
    \begin{cases}
    W(r_{1},r_{2}) = W(r_{2},r_{1}),
    \\
    W(r_{1},r_{3}) \leq W(r_{1},r_{2}) + W(r_{2},r_{3}), 
    \\
    W(r_{1},r_{2}) = 0 \text{ if and only if } r_{1} = r_{2},
    \end{cases}
    \label{Eq2-2}
\end{equation}
where $r_{1},r_{2},r_{3}$ are discrete distributions \cite{villani2008optimal}. The metric property lays the foundation of Wasserstein distance to be used to measure and indicate the dissimilarity between distributions.

However, the Wasserstein distance is usually hindered by its high computational cost. It has been proved that the computational complexity of the Wasserstein distance is at least $O(K^{3}\log (K))$ when comparing two histograms of dimension $K$ \cite{5459199}. To reduce the computational complexity of Wasserstein distance and adapt it to the training procedure of deep neural networks, the Sinkhorn distance is proposed by regularizing the Wasserstein distance with an entropic term \cite{NIPS2013_af21d0c9}. Such regularization empirically reduces the computational complexity of Wasserstein distance to $O(K^{2})$ as well as preserving its metric property. The formulation of the Sinkhorn distance is 

\begin{align}
     W^{\lambda}(r_1,r_2) &= \text{inf}_{P\in \Pi(r_1,r_2)} \langle P,M \rangle - \frac{1}{\lambda}h(P),
     \notag\\
     \text{s.t.}\;\;
     P \mathbf{1}_{K} &= r_2,
     \notag\\
     P^{\top} \mathbf{1}_{K} &= r_1,
     \notag\\
     h(P) &= -\sum_{t_{1},t_{2}}P[t_{1},t_{2}]\log P[t_{1},t_{2}],
    \label{Eq2-5}
\end{align}
where $h(P)$ is the entropic regularization term; $\lambda \in [0,+ \infty]$ is the weight; $t_{1}$ and $t_{2}$ denote the indices of entries in matrix $P$.

\section{Proof of Propositions \ref{Propert-1} and \ref{Propert-2}}
\label{Appendix-A}
In the classification task, the class label is usually represented by the one-hot vector, for example, $y^{k} = (0,...,\underbrace{1}_{k^{th}},...,0)$ is the label of the $k^{th}$ class, and $f_{\theta}(x) \in \mathbb{R}^{K},  f_{\theta}(x)^{\top}\mathbf{1}_{K} = 1$ is the softmax score from the classifier for an arbitrary sample $x$. Determined by these two marginal distributions, there is only one feasible solution of $P$ satisfying the constraint defined in Equation (\ref{Eq3-2}), which is
\begin{align}
    P^{*} = \left[ \mathbf{0}_{K},...,\underbrace{f_{\theta}(x)}_{k^{th}},...,\mathbf{0}_{K}\right].
    \label{Eq3-3}
\end{align}
In this case, the calculation of Wasserstein distance $W(f_{\theta}(x),y^{k})$ can be reduced into
\begin{align}
    W_{M_{\text{Dy}}}(f_{\theta}(x),y^{k}) = \langle P^{*}, M_{\text{Dy}}\rangle.
    \label{Eq3-4}
\end{align}
Substitute the Equations (\ref{Eq3-9}) and (\ref{Eq3-3}) into Equation (\ref{Eq3-4}), we have
    \begin{align}
        W_{M_{\text{Dy}}}(f_{\theta}(x),y^{k}) &= f_{\theta}(x)^{\top}({\mathbf{1}}_{K}-f_{\theta}(x))
        \notag\\
        &= f_{\theta}(x)^{\top}{\mathbf{1}}_{K}-f_{\theta}(x)^{\top}f_{\theta}(x)
        \notag\\
        &= 1-\sum_{i=1}^{K}\left(f_{\theta}(x)[i]\right)^{2},
        \label{EqA-1}
    \end{align}
where $f_{\theta}(x)[i]$ is the $i^{th}$ element in vector $f_{\theta}(x)$.The Equation (\ref{EqA-1}) shows that different labels $y^{k}$ will not influence the value of Wasserstein distance with dynamic distance matrix, which completes the proof of Proposition \ref{Propert-1}

The maximum value of $W_{M_{\text{Dy}}}(f_{\theta}(x),y^{k})$ with respect to (with respect to) $f_{\theta}(x)$ is achieved when the $\sum_{i=1}^{K}\left(f_{\theta}(x)[i]\right)^{2}$ reaches its minimum value under the constraint that $\sum_{i=1}^{K}f_{\theta}(x)[i] = 1$. Following the Cauchy-Schwarz inequality, $\sum_{i=1}^{K}\left(f_{\theta}(x)[i]\right)^{2}$ reaches the minimum value when $f_{\theta}(x)[1] = ... = f_{\theta}(x)[K] = \frac{1}{K}$, which completes the proof of Proposition \ref{Propert-2}.

\section{Proof of Proposition \ref{prop:4-1}}
\label{Appendix-C}
Since $f_i(\theta)$ is differentiable, it is continuous. Therefore for each $\epsilon>0$, there exists a $\delta_i$, such that for any $\theta\in \mathcal{B}\left(\theta_0,\delta\right)$,
\begin{equation}
    \left|f_i(\theta)-f_i(\theta_0)\right|<\epsilon/2
\end{equation}
Given that $f_i(\theta_0)$ are mutually different, without loss of generality, we assume 
\begin{equation}
    f_1(\theta_0)-f_j(\theta_0)<-M,j\geq 2.
\end{equation}
Therefore for any $\theta\in\mathcal{B}\left(\theta_0,\min_i\delta_i\right)$,
\begin{eqnarray}
    f_1(\theta)-f_j(\theta)&\leq&f_1(\theta_0)+\epsilon/2-f_j(\theta_0)+\epsilon/2\nonumber\\
    &\leq& f_1(\theta_0)-f_j(\theta_0)+\epsilon\nonumber\\
    &<& -M+\epsilon.
\end{eqnarray}
By setting $\epsilon=M/n,n\geq 2$, we have:
\begin{equation}
    f_1(\theta)-f_j(\theta)<0,
\end{equation}
and thus
\begin{equation}
    \min_j f_j(\theta)=f_1(\theta).
\end{equation}
Therefore
\begin{equation}
    \frac{\min_j f_j(\theta)-\min_j f_j(\theta_0)}{\theta-\theta_0}=\frac{f_1(\theta)-f_1(\theta_0)}{\theta-\theta_0}.
\end{equation}
Let $n\rightarrow\infty$, we have $\epsilon\rightarrow 0$, and $\delta\rightarrow 0$. Thus we have:
\begin{equation}
    \lim_{\theta\rightarrow\theta_0}\frac{\min_j f_j(\theta)-\min_j f_j(\theta_0)}{\theta-\theta_0}=\nabla f_1(\theta_0).
\end{equation}
The proof of Proposition \ref{prop:4-1} is completed.

\section{Gradient of Wasserstein Distance}
\label{Appendix-E}
The definition of Wasserstein distance in Equation (\ref{Eq3-1}) implies that the Wasserstein distance is a linear programming (LP) and its gradient with respect to $f_{\theta}(x)$ can be computed via Lagrangian duality. The dual formulation of  Equation (\ref{Eq3-1}) is expressed as
\begin{align}
    W^\text{D} (a,b) &= \sup_{a,b} a^{\top}y^{k} + b^{\top}f_{\theta}(x),
    \notag\\
    \text{s.t.} \;\; &a[t_{1}] + b[t_{2}] \leq M[t_{1},t_{2}],
    \notag\\
    & a,b \in \mathbb{R}^{K},
    \notag\\
    & k \in \mathcal{K}_\text{InD}
    \label{Eq4-1}
\end{align}
where $a,b$ are the dual variables, $a[t_{1}]$ is the $t_{1}^{th}$ element in vector $a$, $b[t_{2}]$ is the $t_{2}^{th}$ element in vector $b$, and $M[t_{1},t_{2}]$ is the element in the $t_{1}^{th}$ row and $t_{2}^{th}$ column of $M$. Since the primal problem is an LP, the optimal values of primal and dual problems are equal. So the value of dual variable $b$ at the optimal point is the desired gradient of the Wasserstein distance with respect to $f_{\theta}(x)$ in Equation~(\ref{eq:multiply}). The problem is reduced to derive the value of dual variables at the optimal point. That is, to compute $b^{*}$, which is commonly implemented by solving the Lagrangian dual problem. However, the classic solution to the Langrangian dual problem has a high computational cost \cite{5459199}. 

To improve the computation efficiency, researchers introduced a smoothed primal problem by adding an entropic regularization term and revise the Wasserstein distance into the Sinkhorn distance \cite{NIPS2013_af21d0c9}. The smoothed primal problem is formulated as
\begin{align}
    W(f_{\theta}(x),y^{k^*}) &= \inf_{P} \langle P,M \rangle - \frac{1}{\lambda}h(P),
    \notag\\
    \text{s.t.}\;\;
    h(P) &= -\sum_{t_{1},t_{2}}P[t_{1},t_{2}]\log P[t_{1},t_{2}]
    \notag\\
    P\mathbf{1}_{K} &= y^{k^*},
    \notag\\
    P^{\top}\mathbf{1}_{K} &= f_{\theta}(x).
    \label{Eq4-2}
\end{align}

The Lagrangian of Equation (\ref{Eq4-2}) is
\begin{align}
    W^{L}(P, a, b) &= \langle P,M \rangle - \frac{1}{\lambda}h(P) 
    \notag\\
    & + \langle a, P \mathbf{1}_{K} - y^{k}\rangle + \langle b, P^{\top} \mathbf{1}_{K} - f_{\theta}(x)\rangle.
    \label{Eq4-3}
\end{align}

Taking the derivative of Equation (\ref{Eq4-3}) with respect to $P$, we have
\begin{align}
    \frac{\partial W^{L}(P, a, b)}{\partial P[t_{1},t_{2}]}&= M[t_{1},t_{2}] + \frac{1}{\lambda}(\log(P[t_{1},t_{2}])+1) 
    \notag\\
    & + a[t_{1}] + b[t_{2}].
    \label{Eq4-4}
\end{align}

Hence, the solution of $\frac{\partial W^{L}(P, a, b)}{\partial P[t_{1},t_{2}]} = 0$ is
\begin{align}
    P[t_{1},t_{2}] &= \exp^{-\lambda a[t_{1}]-\frac{1}{2}}\exp^{-\lambda M[t_{1},t_{2}]}\exp^{-\lambda b[t_{2}]-\frac{1}{2}}.
    \label{Eq4-5}
\end{align}

Given the kernel matrix $K = \exp^{-\lambda M}$, vectors $u = \exp^{-\lambda a - \frac{1}{2}}, v = \exp^{-\lambda b - \frac{1}{2}}$, we have the solution $P = \text{diag}(u)K\text{diag}(v)$, and the value of vectors $u,v$ directly determine the dual variables $a,b$. Considering the solution must satisfy the marginalized constraints, we have 

\begin{align}
    &\left\{
    \begin{aligned}\text{diag}(u)K\text{diag}(v) \mathbf{1}_{K} &= y^{k},\\          \text{diag}(v)K^{\top}\text{diag}(u) \mathbf{1}_{K} &= f_{\theta}(x), \end{aligned}\right.
    \notag\\
    &\left\{
    \begin{aligned}u \odot (Kv) &= y^{k},\\          v \odot (K^{\top}u) &= f_{\theta}(x). \end{aligned}\right.
    \label{Eq4-6}
\end{align}

The $u$ and $v$ can be updated iteratively by Sinkhorn-Knopp algorithm \cite{NIPS2013_af21d0c9}, which is given in Algorithm \ref{Alg:Sinkhorn}.  
\begin{algorithm}[H]
  \caption{Sinkhorn-Knopp Algorithm}
  \begin{algorithmic}[1]
    \Inputs{$y^{k}, f_{\theta}(x), \lambda, K$}
    \State{\textbf{while} $v$ not converge \textbf{do}}
    \State\;\;\;\;\;$v \gets f_{\theta}(x)/\left(K^{\top} y^{k}/(Kv)\right)$
    \State{\textbf{end while}}
    \State{$u \gets y^{k}/Kv$}
  \end{algorithmic}
  \label{Alg:Sinkhorn}
\end{algorithm}

Given the value of $v$ when converging as $v^{*}$, the gradient of Wasserstein distance with respect to $f_{\theta}(x)$ is given as

\begin{align}
    \nabla_{f_{\theta}(x)}\,W(f_{\theta}(x), y^{k}) = b^{*} = -\frac{1}{\lambda}(\log v^{*} + \frac{1}{2}).
    \label{Eq4-10}
\end{align}

\section{Proof of Statistical Learning Bound}
\label{Appendix-B}
Let the softmax score $f_{\theta}(x)$ and $y^{k} = (0,...,\underbrace{1}_{k^{th}},...,0)$ be two discrete probability distributions on $K$ classes, and the Kullback-Leibler (KL) divergence between two discrete distributions is defined as,
\begin{equation}
    D(y^{k} ||f_{\theta}(x)) = \sum_{i=1}^{K}y^{k}[i]\log \frac{y^{k}[i]}{f_{\theta}(x)}.
    \label{EqB-1}
\end{equation}

Given that $\lim_{x\rightarrow 0^{+}} x\log x = 0$, because $y^{k}$ is the one-hot label and $f_{\theta}(x)$ is the predicted softmax score, Equation (\ref{EqB-1}) can be reduced into
\begin{equation}
    D(y^{k}||f_{\theta}(x)) = -\log f_{\theta}(x)[k]. 
    \label{EqB-2}
\end{equation}

In the classification task, Equation (\ref{EqB-2}) denotes the cross-entropy loss which is used for correctly classifying the in-distribution samples. The WOOD loss introduced in Section \ref{sec:WOOD Loss} included the cross-entropy loss for InD classification and Wasserstein loss for OOD detection.
\begin{align}
    \mathcal{L}(\mathcal{D}_\text{InD}, \mathcal{D}_\text{OOD}) &= \frac{1}{N_\text{InD}}\sum_{(x_{\text{InD}}, y^{k}) \in \mathcal{D}_\text{InD}}-\log \left(f_{\theta}(x_{\text{InD}})^{\top}y^{k} \right)
    \notag\\
    &-\beta \frac{1}{N_\text{OOD}}\sum_{x_{\text{OOD}} \in \mathcal{D}_\text{OOD}}\min_{k}\inf_{P}\langle P, M\rangle,
    \notag\\
    \text{s.t.}\;\;
    P\mathbf{1}_{K} &= y^{k},
    \notag\\
    P^{\top}\mathbf{1}_{K} &= f_{\theta}(x_{\text{InD}}),
    \notag\\
    k &\in \mathcal{K}_\text{InD}.
    \label{EqB-3}
\end{align}

Suppose we have the independent and identically distributed training samples $\mathcal{D}$, which contains InD and OOD training samples $(x_{\text{InD}}, y^{k}) \in \mathcal{D}_\text{InD}, x_{\text{OOD}} \in \mathcal{D}_\text{OOD}$, respectively, the empirical risk $\hat{R}_{\mathcal{D}}$ and expected risk $R$ is denoted as

\begin{align}
    \hat{R}_{\mathcal{D}}(f_{\theta}) &= \hat{\mathbb{E}}_{\mathcal{D}}[\mathcal{L}(f_{\theta}(x), y^{k})]
    \notag\\
    R(f_{\theta}) &= {\mathbb{E}}[\mathcal{L}(f_{\theta}(x), y^{k})],
    \label{EqB-4}
\end{align}
where $\hat{\mathbb{E}}_{\mathcal{D}}$ denotes the empirical average over the dataset $\mathcal{D}$, and $\mathbb{E}$ represents the expectation. Let ${R}_{1}(f_{\theta}),  \hat{R}_{1, \mathcal{D}_\text{InD}}(f_{\theta})$ denote the expected risk and empirical risk of cross-entropy loss for InD samples, respectively, ${R}_{2}(f_{\theta}), \hat{R}_{2, \mathcal{D}_\text{OOD}}(f_{\theta})$ denote the expected risk and empirical risk of Wasserstein loss for OOD samples, respectively. We have 
\begin{align}
    \hat{R}_{\mathcal{D}}(f_{\theta}) &= \hat{R}_{1, \mathcal{D}_\text{InD}}(f_{\theta}) - \beta\hat{R}_{2, \mathcal{D}_\text{OOD}}(f_{\theta}),
    \notag\\
    R(f_{\theta}) &= {R}_{1}(f_{\theta}) - \beta {R}_{2}(f_{\theta}).
\end{align}

The proof of Theorem \ref{them4-3} can be divided into three steps. First, we prove that the difference between the empirically minimized WOOD loss and the globally minimized WOOD loss is bounded by the summation of uniform concentration bounds of the cross-entropy loss ($l_{1}$) and the Wasserstein loss ($l_{2}$). Second, the uniform concentration bounds of these two losses are derived, respectively. For each of them, McDiarmid's Inequality \cite{10.5555/2011878} is introduced to bound the difference between expected and empirical minimized risks with the Rademacher complexity \cite{10.5555/944919.944944} of the space defined by the loss function ($\mathcal{L}_{1}$ or $\mathcal{L}_{2}$). Third, the Talagrand's Lemma \cite{ledoux1991probability} is further used to bound the Rademacher complexity of $\mathcal{L}_{1}$ or $\mathcal{L}_{2}$ with the Rademacher complexity of the space defined by the classifier ($\mathcal{F}$).  

\begin{lemma}
    Let $\hat\theta$ be the estimated model parameter of the classifier $f$. Let $f_{\hat{\theta}}, f_{\theta^{*}} \in \mathcal{F}$ be the minimizer of the empirical risk $\hat{R}_{\mathcal{D}}$ and expected risk $R$, respectively. Then
    \begin{align*}
        R(f_{\hat{\theta}}) \leq R(f_{\theta^{*}}) &+ 2\sup_{f_{\theta}\in \mathcal{F}} |R_{1}(f_{\theta}) - \hat{R}_{1, \mathcal{D}_\text{InD}}(f_{\theta})| 
        \notag\\
        &+ 2\beta \sup_{f_{\theta}\in \mathcal{F}} |R_{2}(f_{\theta}) -  \hat{R}_{2, \mathcal{D}_\text{OOD}}(f_{\theta})|.
    \end{align*}
    \label{lemB-1}
\end{lemma}

\begin{proof}
Because $f_{\hat{\theta}}$ is the minimizer of $\hat{R}_{\mathcal{D}}$
\begin{align*}
    &R(f_{\hat{\theta}}) - R(f_{\theta^{*}}) 
    \notag\\
    = &R(f_{\hat{\theta}}) - \hat{R}_{\mathcal{D}}(f_{\hat{\theta}}) + \hat{R}_{\mathcal{D}}(f_{\hat{\theta}}) - R(f_{\theta^{*}})
    \notag\\
    \leq &R(f_{\hat{\theta}}) - \hat{R}_{\mathcal{D} }(f_{\hat{\theta}}) + \hat{R}_{\mathcal{D} }(f_{\theta^{*}}) - R(f_{\theta^{*}})
    \notag\\
    \leq &2\sup_{f_{\theta}\in \mathcal{F}} \left|R(f_{\theta}) - \hat{R}_{\mathcal{D}}(f_{\theta})\right|
    \notag\\
    = &2\sup_{f_{\theta}\in \mathcal{F}} \left|R_{1}(f_{\theta}) - \beta R_{2}(f_{\theta}) -  \left(\hat{R}_{1, \mathcal{D}_{\text{InD}}}(f_{\theta}) - \beta \hat{R}_{2, \mathcal{D}_{\text{OOD}}}(f_{\theta})\right)\right|
    \notag\\
    \leq  &2\sup_{f_{\theta}\in \mathcal{F}} \left(\left|R_{1}(f_{\theta}) - \hat{R}_{1, \mathcal{D}_{\text{InD}}}(f_{\theta})\right| + \beta \left|R_{2}(f_{\theta}) -  \hat{R}_{2, \mathcal{D}_{\text{OOD}}}(f_{\theta})\right|\right)
    \notag\\
    \leq  &2\sup_{f_{\theta}\in \mathcal{F}} \left|R_{1}(f_{\theta}) - \hat{R}_{1, \mathcal{D}_{\text{InD}}}(f_{\theta})\right| 
    \notag\\
    &+ 2\beta \sup_{f_{\theta}\in \mathcal{F}} \left|R_{2}(f_{\theta}) -  \hat{R}_{2, \mathcal{D}_{\text{OOD}}}(f_{\theta})\right|
\end{align*}
\end{proof}

Therefore, to derive the upper bound of $R(f_{\hat{\theta}})$, we need to establish the uniform concentration bounds for the cross-entropy loss and Wasserstein loss, respectively. First, we define the space of cross-entropy loss induced by the hypothesis space $\mathcal{F}$ as
\begin{align}
    \mathcal{L}_{1} = \{l_{1\theta}: (x_{\text{InD}}, y^{k}) \rightarrow D(y^{k}||f_{\theta}(x_{\text{InD}})): f_{\theta}\in \mathcal{F}\}.
\end{align}

\begin{theorem}[McDiarmid's Inequality \cite{10.5555/2011878}]. 
Let $\mathcal{D} = {\mathcal{X}_{1}, ..., \mathcal{X}_{N}}$ be a subset of $N$ independent and identically distributed random variables. Assume there exists $C > 0$ such that $f: \mathcal{D} \rightarrow \mathbb{R}$ satisfies the following stability condition 
\begin{align*}
    |f(\mathcal{X}_{1}, ..., \mathcal{X}_{i}, ..., \mathcal{X}_{N}) - f(\mathcal{X}_{1}, ..., \mathcal{X}^{'}_{i}, ..., \mathcal{X}_{N})| \leq C
\end{align*}
for all $i = 1, ..., N$ and any $\mathcal{X}_{1}, ..., \mathcal{X}_{N}, \mathcal{X}_{i}^{'} \in \mathcal{D}$. Then for any $\epsilon > 0$, denoting $f(\mathcal{X}_{1}, ..., \mathcal{X}_{N})$ by $f(\mathcal{D})$, it holds that
\begin{align*}
    \mathbb{P}(f(\mathcal{D}) - \mathbb{E}[f(\mathcal{D})]\geq \epsilon) \leq \exp{(-\frac{2\epsilon^{2}}{NC^{2}})}
\end{align*}
\label{themB-2}
\end{theorem}

\begin{definition}[Rademacher Complexity \cite{10.5555/944919.944944}]
Let $\mathcal{G}$ be a family of mapping from $\mathcal{Z}$ to $\mathbb{R}$, and $S=(z_{1}, ..., z_{N})$ a fixed sample from $\mathcal{Z}$. The empirical Rademacher complexity of $\mathcal{G}$ with respect to $S$ is defined as
\begin{align}
    \hat{\mathcal{R}}_{S}(\mathcal{G}) = \mathbb{E}_{\sigma}\left[\sup_{g\in\mathcal{G}}\frac{1}{N}\sum_{i=1}^{n}\sigma_{i}g(z_{i})\right]
\end{align}
where $\sigma = (\sigma_{1},...,\sigma_{N})$, with $\sigma_{i}$ is independent uniform random variables taking values in $\{+1, -1\}$. $\sigma_{i}$ is called the Rademacher random variables. The Rademacher complexity is defined by taking expectation with respect to the samples $S$,
\begin{align}
    \mathcal{R}_{N}(\mathcal{G}) = \mathbb{E}_{S}[\hat{\mathcal{R}}_{S}(\mathcal{G})]
\end{align}
\label{defB-3}
\end{definition}

\begin{lemma}[Theorem 3 in \cite{vu17353}]
    The upper bound of the KL divergence (\ref{EqB-2}) is 
    \begin{equation}
        D(y^{k}||f_{\theta}(x)) \leq \sum_{i=1}^{K} \frac{(y^{k}[i])^{2}}{f_{\theta}(x)[i]} - 1.
    \end{equation}
    \label{lemB-4}
\end{lemma}

\begin{lemma}
    Let $\mathcal{S} = \{f_{\theta}(x_{1}), ..., f_{\theta}(x_{N_\text{InD}})\}$ be a set of predicted labels of the InD training dataset $\{(x_{1}, y^{k_{1}}), ..., (x_{N_\text{InD}}, y^{k_{N_{\text{InD}}}})\}$, in which $f_{\theta} \in \mathbb{R}^{K}$, and $\{y^{k_{1}}, ..., y^{k_{N_{\text{InD}}}}\}$ are ground-truth one-hot label vectors. Suppose the minimum value of $f_{\theta}(x_{\text{InD}})[i],$ for all $x_{\text{InD}} \in \mathcal{D}_{\text{InD}}, i \in \mathcal{K}$ is $m \in (0,1)$, it holds that 
    \begin{align}
        D(y^{k}||f_{\theta}(x_{\text{InD}})) \leq \frac{1}{m} - 1, \text{ for all } (x_{\text{InD}}, k) \in \mathcal{D}_{\text{InD}}
    \end{align}
    \label{lemB-5}
\end{lemma}
\begin{proof}
From Lemma \ref{lemB-4}, we know that $D(y^{k}||f_{\theta}(x_{\text{InD}})) \leq \sum_{k=1}^{K}\frac{(y^{k}[i])^{2}}{f_{\theta}(x_{\text{InD}})[i]} - 1$. Because $y^{k}$ is the ground-truth label of classification problem, it is an one-hot vector with only one entry as $1$ and all the others are $0$. Assume that the predicted label from the model can not be $0$, we have 
\begin{align}
    D(y^{k}||f_{\theta}(x_{\text{InD}})) &\leq \frac{1}{f_{\theta}(x_{\text{InD}})^{\top}y^{k}} - 1
    \notag\\
    &\leq \frac{1}{m} - 1, \text{ for all } (x_{\text{InD}}, k) \in \mathcal{D}_{\text{InD}}
\end{align}
\end{proof}

\begin{theorem}
For all $\delta\in(0,1)$, with probability at least $1-\delta$, for all $l_{1\theta} \in \mathcal{L}_{1}$, we have 

\begin{align}
    \mathbb{E}[l_{1\theta}] - \hat{\mathbb{E}}_{\mathcal{D}_{\text{InD}}}[l_{1\theta}]&\leq 2\mathcal{R}_{N_{\text{InD}}}(\mathcal{L}_{1}) + \sqrt{\frac{(1-m)^{2}\log (1/\delta)}{2N_{\text{InD}}m^{2}}}  
\end{align}
\end{theorem}

\begin{proof}
For any $l_{1\theta} \in \mathcal{L}_{1}$, the empirical expectation can be reduced into the empirical risk of the corresponding $f_{\theta}$, which is
\begin{align}
    \hat{\mathbb{E}}_{\mathcal{D}_{\text{InD}}}[l_{1\theta}] = \frac{1}{N_{\text{InD}}}\sum_{i=1}^{N_{\text{InD}}}D(y^{k} || f_{\theta}(x_{\text{InD}}) ) = \hat{R}_{1, \mathcal{D}_{\text{InD}}}(f_{\theta}).
\end{align}

We also have $\mathbb{E}[l_{1\theta}] = R_{1}(f_{\theta})$. Let
\begin{align}
    \Phi(\mathcal{D}_{\text{InD}}) = \sup_{l_{1\theta} \in \mathcal{L}_{1}}\mathbb{E}[l_{1\theta}] - \hat{\mathbb{E}}_{\mathcal{D}_{\text{InD}} }[l_{1\theta}].
\end{align}

Let $\mathcal{D}_{\text{InD}}^{'}$ denote the $i^{th}$ sample of $\mathcal{D}_{\text{InD}} $ is replaced by $(x_{\text{InD}}^{'}, y^{k^{'}})$. By Lemma \ref{lemB-5}, we have
\begin{align}
    &\;\;\Phi(\mathcal{D}_{\text{InD}} ) - \Phi(\mathcal{D}_{\text{InD}}^{'}) 
    \notag\\
    \leq & \sup_{l_{1\theta} \in \mathcal{L}_{1}} \hat{\mathbb{E}}_{\mathcal{D}_{\text{InD}}^{'  }}[l_{1\theta}] - \hat{\mathbb{E}}_{\mathcal{D}_{\text{InD}} }[l_{1\theta}]
    \notag\\
    = & \sup_{f_{\theta}\in \mathcal{F}} \frac{D(y^{k^{'}} || f_{\theta}(x_{\text{InD}}^{'})) - D(y^{k} || f_{\theta}(x_{\text{InD}}))}{N_{\text{InD}} }
    \notag\\
    \leq &\frac{1-m}{m N_{\text{InD}} }.
\end{align}

Similarly we have $\Phi(\mathcal{D}_{\text{InD}}^{'}) - \Phi(\mathcal{D}_{\text{InD}} ) \leq \frac{1-m}{m N_{\text{InD}} }$, thus $\left|\Phi(\mathcal{D}_{\text{InD}}^{'}) - \Phi(\mathcal{D}_{\text{InD}} )\right| \leq \frac{1-m}{m N_{\text{InD}} }$. By Theorem \ref{themB-2}, we have 
\begin{align}
    \Phi(\mathcal{D}_{\text{InD}} ) \leq \mathbb{E}[\Phi(\mathcal{D}_{\text{InD}} )] + \sqrt{\frac{(1-m)^{2}\log (1/\delta)}{2N_{\text{InD}}m^{2}}}.
    \label{EqB-5}
\end{align}

Next, we need to bound $\mathbb{E}[\Phi(\mathcal{D}_{\text{InD}})]$. Suppose $\mathcal{D}_{\text{InD}}^{''} = \{(x_{1}^{''}, y^{k_{1}^{''}}),...,(x^{''}_{N_{\text{InD}} }, y^{k_{N_{\text{InD}}}^{''}})\}$ is another sequence of ghost samples, we have
\begin{align}
    &\mathbb{E}_{\mathcal{D}_{\text{InD}} }[\Phi(\mathcal{D}_{\text{InD}} )]
    \notag\\
    =& \mathbb{E}_{\mathcal{D}_{\text{InD}} }\left[\sup_{l_{1\theta} \in \mathcal{L}_{1}}\mathbb{E}[l_{1\theta}] - \hat{\mathbb{E}}_{\mathcal{D}_{\text{InD}} }[l_{1\theta}]\right]
    \notag\\
    =& \mathbb{E}_{\mathcal{D}_{\text{InD}} }\left[\sup_{l_{1\theta} \in \mathcal{L}_{1}}\mathbb{E}_{\mathcal{D}_{\text{InD}}^{'' }}\left[\hat{\mathbb{E}}_{\mathcal{D}_{\text{InD}}^{'' }}[l_{1\theta}] - \hat{\mathbb{E}}_{\mathcal{D}_{\text{InD}} }[l_{1\theta}]\right]\right]
    \notag\\
    \leq & \mathbb{E}_{\mathcal{D}_{\text{InD}} , \mathcal{D}_{\text{InD}}^{''}}\left[\sup_{l_{1\theta} \in \mathcal{L}_{1}}\hat{\mathbb{E}}_{\mathcal{D}_{\text{InD}}^{''}}[l_{1\theta}] - \hat{\mathbb{E}}_{\mathcal{D}_{\text{InD}} }[l_{1\theta}]\right].
    \label{EqB-6}
\end{align}
We further examine the difference of empirical averages, $\hat{\mathbb{E}}_{\mathcal{D}_{\text{InD}}^{''}}[l_{1\theta}] - \hat{\mathbb{E}}_{\mathcal{D}_{\text{InD}} }[l_{1\theta}]$. Suppose we have two new sets, $\mathcal{S}$ and $\mathcal{S}^{'}$, in which the $i^{th}$ data points in sets $\mathcal{D}_{\text{InD}} $ and $\mathcal{D}_{\text{InD}}^{'' }$ are swapped with the probability of $\frac{1}{2}$. We have $\hat{\mathbb{E}}_{\mathcal{D}_{\text{InD}}^{'' }}[l_{1\theta}] - \hat{\mathbb{E}}_{\mathcal{D}_{\text{InD}} }[l_{1\theta}]$ and $\hat{\mathbb{E}}_{\mathcal{S}^{'}}[l_{1\theta}] - \hat{\mathbb{E}}_{\mathcal{S}}[l_{1\theta}]$ has the same distribution, which is because all the samples are independent and identically distributed and permutation does not change the distribution. So that we have
\begin{align}
    &\hat{\mathbb{E}}_{\mathcal{D}_{\text{InD}}^{'' }}[l_{1\theta}] - \hat{\mathbb{E}}_{\mathcal{D}_{\text{InD}} }[l_{1\theta}] 
    \notag\\
    = &\frac{1}{N_{\text{InD}} }\sum_{i=1}^{N_{\text{InD}} }\left(l_{1\theta}(x_{\text{InD}}^{''}, y^{k^{''}}) - l_{1\theta}(x_{\text{InD}}, y^{k})\right),
\notag\\
    &\hat{\mathbb{E}}_{\mathcal{S}^{'}}[l_{1\theta}] - \hat{\mathbb{E}}_{\mathcal{S}}[l_{1\theta}] 
    \notag\\
    = &\frac{1}{N_{\text{InD}} }\sum_{i=1}^{N_{\text{InD}} }\sigma_{i}\left(l_{1\theta}(x_{\text{InD}}^{''}, y^{k^{''}}) - l_{1\theta}(x_{\text{InD}}, y^{k})\right),
    \label{EqB-7}
\end{align}

where $\sigma_{i}$ is the Rademacher variables introduced in Definition \ref{defB-3}, which means the $i^{th}$ samples in sets $\mathcal{S}^{'}$ and $\mathcal{S}$ are swapped with the probability of $\frac{1}{2}$. Substitute the Equation (\ref{EqB-7}) into Equation (\ref{EqB-6}), we have
\begin{align}
    &\mathbb{E}_{\mathcal{D}_{\text{InD}} }[\Phi(\mathcal{D}_{\text{InD}} )] 
    \notag\\
    \leq &\mathbb{E}_{\mathcal{D}_{\text{InD}} , \mathcal{D}_{\text{InD}}^{''}, \sigma}\left[\sup_{l_{1\theta} \in \mathcal{L}_{1}}\frac{\sum_{i=1}^{N_{\text{InD}} }\sigma_{i}\left(l_{1\theta}(x_{\text{InD}}^{''}, y^{k^{''}}) - l_{1\theta}(x_{\text{InD}}, y^{k})\right)}{N_{\text{InD}} }\right]
    \notag\\
    \leq &\mathbb{E}_{\mathcal{D}_{\text{InD}}^{''}, \sigma}\left[\sup_{l_{1\theta} \in \mathcal{L}_{1}}\frac{1}{N_{\text{InD}} }\sum_{i=1}^{N_{\text{InD}} }\sigma_{i}l_{1\theta}(x_{\text{InD}}^{''}, y^{k^{''}})\right] 
    \notag\\
    &+ \mathbb{E}_{\mathcal{D}_{\text{InD}} , \sigma}\left[\sup_{l_{1\theta} \in \mathcal{L}_{1}}\frac{1}{N_{\text{InD}} }\sum_{i=1}^{N_{\text{InD}} }-\sigma_{i}l_{1\theta}(x_{\text{InD}}, y^{k})\right]
    \notag\\
    = &\mathbb{E}_{\mathcal{D}_{\text{InD}}^{'' }}[\hat{\mathcal{R}}_{\mathcal{D}_{\text{InD}}^{'' }}(\mathcal{\mathcal{L}}_{1})] + \mathbb{E}_{\mathcal{D}_{\text{InD}} }[\hat{\mathcal{R}}_{\mathcal{D}_{\text{InD}} }(\mathcal{\mathcal{L}}_{1})]
    \notag\\
    = &2\mathcal{R}_{N_{\text{InD}} }(\mathcal{L}_{1}).
    \label{EqB-8}
\end{align}
By combining Equations (\ref{EqB-5}, \ref{EqB-8}), we have 
\begin{align}
    \Phi(\mathcal{D}_{\text{InD}} ) &\leq \mathbb{E}[\Phi(\mathcal{D}_{\text{InD}} )] + \sqrt{\frac{(1-m)^{2}\log (1/\delta)}{2N_{\text{InD}} m^{2}}}
    \notag\\
    &\leq 2\mathcal{R}_{N_{\text{InD}} }(\mathcal{L}_{1}) + \sqrt{\frac{(1-m)^{2}\log (1/\delta)}{2N_{\text{InD}} m^{2}}}.
    \label{EqB-9}
\end{align}
\end{proof}

To this point, we have bounded the difference between expected and empirical risks using the Rademacher complexity of the family of cross-entropy loss $\mathcal{L}_{1}$ and a constant. We further try to bound the Rademacher complexity of $\mathcal{L}_{1}$ with the Rademacher complexity of the hypothesis class $\mathcal{F}$. 
\begin{lemma}[Talagrand's Lemma \cite{ledoux1991probability}]
    Let $\mathcal{F}$ be a class of real functions, If $l: \mathbb{R}^{K} \rightarrow \mathbb{R}$ is a $L_{l}$-Lipschitz function and $l(0) = 0$, then $\mathcal{R}_{N}(l \circ \mathcal{F}) \leq L_{l}\mathcal{R}_{N}(\mathcal{F})$
    \label{lemB-7}
\end{lemma}

Consider a family of loss functions $\mathcal{L}_{1} = \{z \rightarrow l_{1}(f_{\theta}(x), y): f_{\theta}\in\mathcal{F}\}$, in our case, $l_{1}$ represents the cross-entropy loss, $f_{\theta}$ represents the neural network with softmax score as the output (eliminate 0 elements in the output).

\begin{proposition}
For all $(x_{1}, y^{k_{1}}), (x_{2}, y^{k_{2}}) \in \mathcal{D}_{\text{InD}}$, the cross-entropy loss defined by $l_{1\theta}(f_{\theta}(x_{\text{InD}}), y^{k}) = -\log \left(f_{\theta}(x_{\text{InD}})^{\top}y^{k} \right)$ satisfies
\begin{align}
    &||l_{1\theta}(f_{\theta}( x_{1}), y^{k_{1}}) - l_{1\theta}(f_{\theta}( x_{2}), y^{k_{2}})||_{2} 
    \notag\\
    \leq &\frac{\sqrt{2}}{m} ||(f_{\theta}( x_{1}), y^{k_{1}}) - (f_{\theta}( x_{2}), y^{k_{2}})||_{2}.
\end{align}
\label{propB-8}
\end{proposition}

\begin{proof}
The cross-entropy loss is defined as
\begin{align}
    l_{1\theta}(f_{\theta}( x_{\text{InD}}), y^{k}) &= -\log \left(f_{\theta}( x_{\text{InD}})^{\top}y^{k} \right)
    \notag\\
    f_{\theta}( x_{\text{InD}})[i] &=\frac{\exp ({g}_{\theta}( x_{\text{InD}})[i])}{\sum_{j=1}^{K}\exp({g}_{\theta}( x_{\text{InD}})[j])}, 
    \label{EqB-10}
\end{align}
where $f_{\theta}( x_{\text{InD}})\in \mathbb{R}^{K}$ represents the output of softmax function, which is used to eliminate the possible $0$ elements and transform the raw output from the neural network ${g}_{\theta}( x_{\text{InD}})$ into discrete distribution.

We would like to dervie the Lipschitz constant $L_{l}$ for the cross-entropy loss, which satisfies
\begin{align}
    ||l_{1\theta}(f_{\theta}( x_{1}), y^{k_{1}}) - l_{1\theta}(f_{\theta}( x_{2}), y^{k_{2}})||_{2} 
    \notag\\
    \leq L_{l} ||(f_{\theta}( x_{1}), y^{k_{1}}) - (f_{\theta}( x_{2}), y^{k_{2}})||_{2}.
\end{align}

The value of $L_{l}$ can be expressed as 
\begin{align}
    L_{l} &= \sup_{\scriptsize \begin{aligned}
    f_{\theta}(x_{1}), y^{k_{1}}\\
    f_{\theta}(x_{2}), y^{k_{2}}
    \end{aligned}} \frac{||l_{1\theta}(f_{\theta}(x_{1}), y^{k_{1}}) - l_{1\theta}(f_{\theta}(x_{2})), y^{k_{2}})||_{2}}{||(f_{\theta}(x_{1}), y^{k_{1}}) - (f_{\theta}(x_{2}), y^{k_{2}})||_{2}}
    \notag\\
    & \leq \sup_{\scriptsize \begin{aligned}
    f_{\theta}(x_{1}), y^{k_{1}}\\
    f_{\theta}(x_{2}), y^{k_{2}}
    \end{aligned}}\frac{||l_{1\theta}(f_{\theta}(x_{1}), y^{k_{1}}) - l_{1\theta}(f_{\theta}(x_{2})), y^{k_{2}})||_{2}}{\frac{\sqrt{2}}{2}(||f_{\theta}( x_{1}) - f_{\theta}( x_{2})||_{2}^{2} + ||y^{k_{1}} - y^{k_{2}}||_{2}^{2})^{\frac{1}{2}}}
    \notag\\
    & \leq \sup_{\scriptsize \begin{aligned}
    f_{\theta}(x_{1}), y^{k_{1}}\\
    f_{\theta}(x_{2}), y^{k_{2}}
    \end{aligned}}\frac{||l_{1\theta}(f_{\theta}(x_{1}), y^{k_{1}}) - l_{1\theta}(f_{\theta}(x_{2})), y^{k_{2}})||_{2}}{\frac{\sqrt{2}}{2}||f_{\theta}( x_{1}) - f_{\theta}( x_{2})||_{2}}
    \notag\\
    & \leq \sqrt{2} \sup_{(f_{\theta}( x_{\text{InD}}), y^{k})}||\nabla_{f_{\theta}( x_{\text{InD}})}l_{1\theta}(f_{\theta}( x_{\text{InD}}), y^{k})||_{2}.
\end{align}

The gradient of $l_{1\theta}(f_{\theta}( x_{\text{InD}}), y^{k})$ with respect to $f_{\theta}( x_{\text{InD}})$ is
\begin{align}
    \nabla_{f_{\theta}( x_{\text{InD}})}l_{1\theta}(f_{\theta}( x_{\text{InD}}), y^{k}) &= \frac{\partial l_{1\theta}}{\partial f_{\theta}( x_{\text{InD}})}
    \notag\\
    &= - \left(\frac{1}{f_{\theta}( x_{\text{InD}})}\right)^{\top}y^{k} .
\end{align}

So that we have 
\begin{align}
    ||\nabla_{f_{\theta}(x_{\text{InD}})}l_{1\theta}(f_{\theta}(x_{\text{InD}}), y^{k})||_{2} \leq \frac{1}{\min_{x_{\text{InD}}, i} f_{\theta}(x_{\text{InD}})[i]} = \frac{1}{m}
\end{align}
\end{proof}
Combine the Lemma \ref{lemB-7} and the Proposition \ref{propB-8}, for the family of loss functions $\mathcal{L}_{1} = \{z \rightarrow l_{1}(f_{\theta}(x_{\text{InD}}), y^{k}): f_{\theta}\in\mathcal{F}\}$, we have $\mathcal{R}_{N_{\text{InD}} }(\mathcal{L}_{1}) = \mathcal{R}_{N_{\text{InD}} }(l_{1} \circ \mathcal{F}) \leq \frac{\sqrt{2}}{m}\mathcal{R}_{N_{\text{InD}} }(\mathcal{F})$. Substitute it into Equation (\ref{EqB-9}), we have
\begin{align}
    &\sup_{l_{1\theta} \in \mathcal{L}_{1}}\mathbb{E}[l_{1\theta}] - \hat{\mathbb{E}}_{\mathcal{D}_{\text{InD}} }[l_{1\theta}]
    \notag\\
    \leq &\frac{2\sqrt{2}}{m}\mathcal{R}_{N_{\text{InD}} }(\mathcal{F}) + \sqrt{\frac{(1-m)^{2}\log (1/\delta)}{2N_{\text{InD}} m^{2}}}
    \label{EqC-1}
\end{align}

To this point, we derived the uniform concentration bound of the cross-entropy loss. Then, we switch to the Wasserstein loss. Similarly, we at first define the space of Wasserstein loss induced by the hypothesis space $\mathcal{F}$ as
\begin{align}
    \mathcal{L}_{2} = \{l_{2\theta}: x_{\text{OOD}} \rightarrow \min_{k} W(f_{\theta}(x_{\text{OOD}}), y^{k}): f_{\theta}\in \mathcal{F}\}.
\end{align}

\begin{lemma}
    Suppose the constant $\alpha_{M} = \max M$ represents the maximum element in all distance matrices $M$, we have $0 \leq \min_{k}W(f_{\theta}(x_{\text{OOD}}), y^{k}) \leq \alpha_{M}$.
    \label{lemB-9}
\end{lemma}

\begin{proof}
Suppose we have $P^{*} \in \Pi(f_{\theta}(x_{\text{OOD}}),y^{k})$ is the optimal solution of $W(f_{\theta}(x_{\text{OOD}}), y^{k})$, we have
\begin{align}
    \min_{k}W(f_{\theta}(x_{\text{OOD}}), y^{k}) \leq &W(f_{\theta}(x_{\text{OOD}}), y^{k})
    \notag\\
    \leq &\alpha_{M} \sum P^{*}
    \notag\\
    = &\alpha_{M}
\end{align}
\end{proof}

With Lemma \ref{lemB-9}, we have the uniform control of the difference between the empirical risk and the expected risk of Wasserstein loss.

\begin{theorem}
For all $\delta \in (0,1)$, with probability at least $1-\delta$, for all $l_{2\theta}\in \mathcal{L}_{2}$, we have 
\begin{align}
    \mathbb{E}[l_{2\theta}] - \hat{\mathbb{E}}_{\mathcal{D}_{\text{OOD}} }[l_{2\theta}]&\leq 2\mathcal{R}_{N_{\text{OOD}} }(\mathcal{L}_{2}) + \sqrt{\frac{\alpha_{M}^{2} \log (1/\delta)}{2N_{\text{OOD}} }}  
\end{align}
\end{theorem}

\begin{proof}
For any $l_{2\theta} \in \mathcal{L}_{2}$, the empirical expectation can be reduced into the empirical risk of the corresponding $f_{\theta}$, which is
\begin{align}
    \hat{\mathbb{E}}_{\mathcal{D}_{\text{OOD}} }[l_{2\theta}] = \frac{1}{N_{\text{OOD}} }\sum_{i=1}^{N_{\text{OOD}} }\min_{k} W(f_{\theta}(x_{\text{OOD}}), y^{k}) = \hat{R}_{2, \mathcal{D}_{\text{OOD}} }(f_{\theta}).
\end{align}

We also have $\mathbb{E}[l_{2\theta}] = R_{2}(f_{\theta})$. Let
\begin{align}
    \Phi(\mathcal{D}_{\text{OOD}} ) = \sup_{l_{2\theta} \in \mathcal{L}_{2}}\mathbb{E}[l_{2\theta}] - \hat{\mathbb{E}}_{\mathcal{D}_{\text{OOD}} }[l_{2\theta}].
\end{align}

Let $\mathcal{D}_{\text{OOD}}^{'  }$ denote the $i^{th}$ sample of $\mathcal{D}_{\text{OOD}} $ is replaced by $x_{\text{OOD}}^{'}$, we have
\begin{align}
    &\;\;\Phi(\mathcal{D}_{\text{OOD}} ) - \Phi(\mathcal{D}_{\text{OOD}}^{'  }) 
    \notag\\
    \leq & \sup_{l_{2\theta} \in \mathcal{L}_{2}} \hat{\mathbb{E}}_{\mathcal{D}_{\text{OOD}}^{'  }}[l_{2\theta}] - \hat{\mathbb{E}}_{\mathcal{D}_{\text{OOD}} }[l_{2\theta}]
    \notag\\
    = & \sup_{f_{\theta}\in \mathcal{F}} \frac{\min_{k} W(f_{\theta}(x_{\text{OOD}}^{'}), y^{k}) - \min_{k} W(f_{\theta}(x_{\text{OOD}}), y^{k})}{N_{\text{OOD}} }
    \notag\\
    \leq &\frac{\alpha_{M}}{ N_{\text{OOD}} }.
\end{align}

Similarly we have $\Phi(\mathcal{D}_{\text{OOD}}^{'  }) - \Phi(\mathcal{D}_{\text{OOD}} ) \leq \frac{\alpha_{M}}{ N_{\text{OOD}} }$, thus $\left|\Phi(\mathcal{D}_{\text{OOD}}^{'  }) - \Phi(\mathcal{D}_{\text{OOD}} )\right| \leq \frac{\alpha_{M}}{ N_{\text{OOD}} }$. By Theorem \ref{themB-2}, we have 
\begin{align}
    \Phi(\mathcal{D}_{\text{OOD}} ) \leq \mathbb{E}[\Phi(\mathcal{D}_{\text{OOD}} )] + \sqrt{\frac{\alpha_{M}^{2} \log (1/\delta)}{2N_{\text{OOD}} }}.
    \label{EqB-11}
\end{align}

Next, we need to bound $\mathbb{E}[\Phi(\mathcal{D}_{\text{OOD}} )]$, suppose $\mathcal{D}_{\text{OOD}}^{''  } = \{x_{1}^{''},...,x^{''}_{N_{\text{OOD}} }\}$ is another sequence of ghost samples, we have
\begin{align}
    &\mathbb{E}_{\mathcal{D}_{\text{OOD}} }[\Phi(\mathcal{D}_{\text{OOD}} )]
    \notag\\
    =& \mathbb{E}_{\mathcal{D}_{\text{OOD}} }\left[\sup_{l_{2\theta} \in \mathcal{L}_{2}}\mathbb{E}[l_{2\theta}] - \hat{\mathbb{E}}_{\mathcal{D}_{\text{OOD}} }[l_{2\theta}]\right]
    \notag\\
    =& \mathbb{E}_{\mathcal{D}_{\text{OOD}} }\left[\sup_{l_{2\theta} \in \mathcal{L}_{2}}\mathbb{E}_{\mathcal{D}_{\text{OOD}}^{'' }}\left[\hat{\mathbb{E}}_{\mathcal{D}_{\text{OOD}}^{'' }}[l_{2\theta}] - \hat{\mathbb{E}}_{\mathcal{D}_{\text{OOD}} }[l_{2\theta}]\right]\right]
    \notag\\
    \leq & \mathbb{E}_{\mathcal{D}_{\text{OOD}} , \mathcal{D}_{\text{OOD}}^{''}}\left[\sup_{l_{2\theta} \in \mathcal{L}_{2}}\hat{\mathbb{E}}_{\mathcal{D}_{\text{OOD}}^{''}}[l_{2\theta}] - \hat{\mathbb{E}}_{\mathcal{D}_{\text{OOD}} }[l_{2\theta}]\right].
    \label{EqB-12}
\end{align}

We further examine the difference of empirical averages, $\hat{\mathbb{E}}_{\mathcal{D}_{\text{OOD}}^{''}}[l_{2\theta}] - \hat{\mathbb{E}}_{\mathcal{D}_{\text{OOD}} }[l_{2\theta}]$. Suppose we have two new sets, $\mathcal{S}$ and $\mathcal{S}^{'}$, in which the $i^{th}$ data points in sets $\mathcal{D}_{\text{OOD}} $ and $\mathcal{D}_{\text{OOD}}^{''}$ are swapped with the probability of $\frac{1}{2}$. We have $\hat{\mathbb{E}}_{\mathcal{D}_{\text{OOD}}^{''}}[l_{2\theta}] - \hat{\mathbb{E}}_{\mathcal{D}_{\text{OOD}} }[l_{2\theta}]$ and $\hat{\mathbb{E}}_{\mathcal{S}^{'}}[l_{2\theta}] - \hat{\mathbb{E}}_{\mathcal{S}}[l_{2\theta}]$ has the same distribution, which is because all the samples are independent and identically distributed and permutation does not change the distribution. So that we have
\begin{align}
    &\hat{\mathbb{E}}_{\mathcal{D}_{\text{OOD}}^{''}}[l_{2\theta}] - \hat{\mathbb{E}}_{\mathcal{D}_{\text{OOD}} }[l_{2\theta}] 
    \notag\\
    &= \frac{1}{N_{\text{OOD}} }\sum_{j=1}^{N_{\text{OOD}} }\left(l_{2\theta}(x_{\text{OOD}}^{''}) - l_{2\theta}(x_{\text{OOD}})\right),
\notag\\
    &\hat{\mathbb{E}}_{\mathcal{S}^{'}}[l_{2\theta}] - \hat{\mathbb{E}}_{\mathcal{S}}[l_{2\theta}] 
    \notag\\
    &= \frac{1}{N_{\text{OOD}} }\sum_{j=1}^{N_{\text{OOD}} }\sigma_{j}\left(l_{2\theta}(x_{\text{OOD}}^{''}) - l_{2\theta}(x_{\text{OOD}})\right),
    \label{EqB-13}
\end{align}

where $\sigma_{j}$ is the Rademacher variables introduced in Definition \ref{defB-3}, which means the $i^{th}$ samples in sets $\mathcal{S}^{'}$ and $\mathcal{S}$ are swapped with the probability of $\frac{1}{2}$. Substitute the Equation (\ref{EqB-13}) into Equation (\ref{EqB-12}), we have
\begin{align}
    &\mathbb{E}_{\mathcal{D}_{\text{OOD}} }[\Phi(\mathcal{D}_{\text{OOD}} )] 
    \notag\\
    \leq &\mathbb{E}_{\mathcal{D}_{\text{OOD}} , \mathcal{D}_{\text{OOD}}^{''}, \sigma}\left[\sup_{l_{2\theta} \in \mathcal{L}_{2}}\frac{\sum_{i=1}^{N_{\text{OOD}} }\sigma_{j}\left(l_{2\theta}(x_{\text{OOD}}^{''}) - l_{2\theta}(x_{\text{OOD}})\right)}{N_{\text{OOD}} }\right]
    \notag\\
    \leq &\mathbb{E}_{\mathcal{D}_{\text{OOD}}^{''}, \sigma}\left[\sup_{l_{2\theta} \in \mathcal{L}_{2}}\frac{1}{N_{\text{OOD}} }\sum_{j=1}^{N_{\text{OOD}} }\sigma_{j}l_{2\theta}(x_{\text{OOD}}^{''})\right] 
    \notag\\
    &+ \mathbb{E}_{\mathcal{D}_{\text{OOD}} , \sigma}\left[\sup_{l_{2\theta} \in \mathcal{L}_{2}}\frac{1}{N_{\text{OOD}} }\sum_{j=1}^{N_{\text{OOD}} }-\sigma_{j}l_{2\theta}(x_{\text{OOD}})\right]
    \notag\\
    = &\mathbb{E}_{\mathcal{D}_{\text{OOD}}^{''}}[\hat{\mathcal{R}}_{\mathcal{D}_{\text{OOD}}^{''}}(\mathcal{\mathcal{L}}_{2})] + \mathbb{E}_{\mathcal{D}_{\text{OOD}} }[\hat{\mathcal{R}}_{\mathcal{D}_{\text{OOD}} }(\mathcal{\mathcal{L}}_{2})]
    \notag\\
    = &2\mathcal{R}_{N_{\text{OOD}} }(\mathcal{L}_{2}).
    \label{EqB-14}
\end{align}
By combining Equations (\ref{EqB-11}, \ref{EqB-14}), we have 
\begin{align}
    \Phi(\mathcal{D}_{\text{OOD}} ) &\leq \mathbb{E}[\Phi(\mathcal{D}_{\text{OOD}} )] + \sqrt{\frac{\alpha_{M}^{2}\log (1/\delta)}{2N_{\text{OOD}} }}
    \notag\\
    &\leq 2\mathcal{R}_{N_{\text{OOD}} }(\mathcal{L}_{2}) + \sqrt{\frac{\alpha_{M}^{2}\log (1/\delta)}{2N_{\text{OOD}} }}.
    \label{EqB-15}
\end{align}
\end{proof}

\begin{proposition}[Proposition B.10 of \cite{NIPS2015_a9eb8122}]
For all $x_{1}, x_{2} \in \mathcal{D}_{\text{OOD}}$, the Wasserstein loss defined by $l_{2\theta}(f(x_{\text{OOD}})) = \min_{k}W(f(x_{\text{OOD}}), y^{k})$ satisfies
\begin{align}
    &||l_{1\theta}(f_{\theta}(x_{1})) - l_{1\theta}(f_{\theta}(x_{2}))||_{2} 
    \notag\\
    \leq &4\alpha_{M}||f_{\theta}(x_{1}) - f_{\theta}(x_{2})||_{2}.
\end{align}
\label{propB-11}
\end{proposition}

Combine the Lemma \ref{lemB-7} and the Proposition \ref{propB-11}, for the family of loss functions $\mathcal{L}_{2} = \{l_{2\theta}: x_{\text{OOD}} \rightarrow \min_{k} W(f_{\theta}(x_{\text{OOD}}), y^{k}): f_{\theta}\in \mathcal{F}\}$, we have $\mathcal{R}_{N_{\text{OOD}} }(\mathcal{L}_{2}) = \mathcal{R}_{N_{\text{OOD}} }(l_{2} \circ \mathcal{F}) \leq 4\alpha_{M}\mathcal{R}_{N_{\text{OOD}} }(\mathcal{F})$. Substitute it into Equation (\ref{EqB-15}), we have
\begin{align}
    &\sup_{l_{2\theta} \in \mathcal{L}_{2}}\mathbb{E}[l_{2\theta}] - \hat{\mathbb{E}}_{\mathcal{D}_{\text{OOD}} }[l_{2\theta}]
    \notag\\
    \leq &8\alpha_{M}\mathcal{R}_{N_{\text{OOD}} }(\mathcal{F}) + \sqrt{\frac{\alpha_{M}^{2}\log (1/\delta)}{2N_{\text{OOD}} }}
    \label{EqC-2}
\end{align}

The proof of Theorem \ref{them4-3} is completed by combining Lemma \ref{lemB-1}, Equations (\ref{EqC-1}), and (\ref{EqC-2}).

% You can push biographies down or up by placing
% a \vfill before or after them. The appropriate
% use of \vfill depends on what kind of text is
% on the last page and whether or not the columns
% are being equalized.

%\vfill

% Can be used to pull up biographies so that the bottom of the last one
% is flush with the other column.
%\enlargethispage{-5in}

% that's all folks
\end{document}